\documentclass{article}

\usepackage[margin=1in]{geometry}
\usepackage{amsmath}
\usepackage{amssymb}
\usepackage{amsfonts}
\usepackage{amsthm}
\usepackage{mathrsfs}
\usepackage{mathtools}
\usepackage{bm}
\usepackage{xcolor}
\usepackage{ifthen}
\usepackage{xifthen}
\usepackage{graphicx}
\usepackage{booktabs}
\usepackage{multirow}
\usepackage{tabularx}
\usepackage{subcaption}
\usepackage{enumerate}
\usepackage{algorithm}
\usepackage{algorithmic}
\usepackage{tikz}
\usepackage[sort&compress,numbers]{natbib}
\setcitestyle{numbers,square}
\usepackage[colorlinks,
            linkcolor=red,
            anchorcolor=blue,
            citecolor=magenta
            ]{hyperref}
\usepackage{cleveref}

\graphicspath{{./}{../}}

\newtheorem{theorem}{Theorem}[section]
\newtheorem{assumption}[theorem]{Assumption}

\newtheorem{proposition}[theorem]{Proposition}
\newtheorem{corollary}[theorem]{Corollary}
\newtheorem{lemma}[theorem]{Lemma}
\newtheorem{definition}[theorem]{Definition}

\let\E\relax
\def\method{\textsc{AD-Seq}}

\newcommand{\R}{\mathbb{R}}

\newcommand{\E}{\mathbb{E}}
\renewcommand{\P}{\mathbb{P}}
\newcommand{\cR}{\mathcal{R}}
\newcommand{\cF}{\mathcal{F}}
\newcommand{\cN}{\mathcal{N}}
\newcommand{\cO}{\mathcal{O}}
\newcommand{\cD}{\mathcal{D}}
\newcommand{\cH}{\mathcal{H}}

\newcommand{\cL}{\mathcal{L}}
\newcommand{\cG}{\mathcal{G}}
\newcommand{\cS}{\mathcal{S}}
\newcommand{\bSigma}{{\boldsymbol\Sigma}}
\newcommand{\bGamma}{{\boldsymbol\Gamma}}
\newcommand{\bs}{\mathbf{s}}
\newcommand{\bx}{\mathbf{x}}
\newcommand{\by}{\mathbf{y}}
\newcommand{\bz}{\mathbf{z}}
\newcommand{\bv}{\mathbf{v}}
\newcommand{\bb}{\mathbf{b}}
\newcommand{\bmu}{{\boldsymbol\mu}}
\newcommand{\btheta}{{\boldsymbol\theta}}

\newcommand{\bzero}{\mathbf{0}}

\newcommand{\bX}{\mathbf{X}}
\newcommand{\bA}{\mathbf{A}}
\newcommand{\bI}{\mathbf{I}}
\newcommand{\bM}{\mathbf{M}}
\newcommand{\bS}{\mathbf{S}}

\newcommand{\EE}[2][]{\ifthenelse{\isempty{#1}}{\mathbb{E}\left[#2\right]}{\mathbb{E}_{#1}\left[#2\right]}}

\def \dd{{\rm d}}
\newcommand{\abs}[1]{\left |#1\right |}
\newcommand{\norm}[1]{\left \|#1\right \|}

\newcommand{\curly}[1]{\left\lbrace #1 \right\rbrace}
\newcommand{\bracket}[1]{\left[ #1 \right]}
\newcommand{\parenthesis}[1]{\left( #1 \right)}

\makeatletter
\@ifundefined{if@BLINDREV}{\newif\if@BLINDREV\@BLINDREVfalse}{}
\newif\ifadseqblindreview
\if@BLINDREV
  \adseqblindreviewtrue
\else
  \adseqblindreviewfalse
\fi
\makeatother

\title{Diffusion Models for Adaptive Sequential Data~Generation}
\author{Haoyang Cao\thanks{Department of Applied Mathematics and Statistics, Data Science and AI Institute, and Mathematical Institute for Data Science, Johns Hopkins University. \textbf{Email:} hycao@jhu.edu}
\and
Minshuo Chen\thanks{Department of Industrial Engineering and Management Sciences,
  Northwestern University. \textbf{Email:} minshuo.chen@northwestern.edu}
\and
Yinbin Han\thanks{Department Management Science and Engineering,
  Stanford University. \textbf{Email:} yinbinha@stanford.edu}
\and
Renyuan Xu\thanks{Department Management Science and Engineering,
  Stanford University. \textbf{Email:} renyuanxu@stanford.edu}}
\date{Jun 2026}

\begin{document}
\maketitle
\allowdisplaybreaks

\begin{abstract}
Generating realistic synthetic sequential data is critical in real-world applications across operations research, finance, healthcare, energy systems, and scientific computing, where time-indexed observations are used for prediction, simulation, risk assessment, and data-driven decision-making. While diffusion models have achieved remarkable success in generating {\it static} data, their direct extensions to sequential settings often fail to capture temporal dependence and information structure. Designing diffusion models that can simulate sequential data in an {\it adapted} manner, and hence without anticipation of future information, therefore remains an {\it open} challenge.

In this work, we propose a sequential forward–backward diffusion framework for adapted time series generation. Our approach progressively injects and removes noise along the sequence,  conditioning on the previously generated history to ensure adaptiveness. A novel score-matching objective is introduced for efficient parallel training. We derive rigorous statistical guarantees under a generic framework, then establish score approximation, score estimation, and distribution estimation results with ReLU networks serving as a concrete instance. Empirically, we validate our method on synthetic data, including ARMA models and Gaussian processes, and demonstrate its effectiveness in constructing mean-variance optimal portfolios.  Our code is available at \ifadseqblindreview\url{https://anonymous.4open.science/r/adapted_diffusion_model-0C88}\else\url{https://github.com/yinbinhan/adapted_diffusion_model}\fi.
\end{abstract}

\section{Introduction}

Generating realistic synthetic sequential data is important for many real-world applications across operations research, financial markets, healthcare, energy systems, and scientific computing, where time-indexed observations are used for prediction, simulation, risk assessment, and data-driven decision-making. At the same time, this task remains challenging in complex sequential settings, where the data may exhibit temporal dependence, nonstationarity, nonlinear dynamics, and multiple sources of uncertainty. Beyond matching the high-dimensional joint law of the observed process, a sequential generator must also respect the information flow under which the data are observed: each generated value should depend only on the past and present information available at that stage, and not on future observations. This adapted (i.e., non-anticipative) structure is a basic modeling requirement in domains where the timing of information affects prediction or decision-making.
For example, in stress testing, synthetic data generation requires not only capturing temporal dependence, but also preserving adaptiveness with respect to the prescribed information flow \citep{acciaio2024time,alouadi2025robust,cont2025tail}; in portfolio optimization, trading decisions are updated dynamically based only on currently available signals \citep{zhou2000continuous,tsoukalas2019dynamic}; and in simulation-based policy learning, synthetic data are used to train or evaluate decision rules that themselves operate sequentially \citep{jia2023q,gao2025data,Coletta2021Trm,huang2024mean}.  Under these circumstances, preserving the temporal direction of information is as essential as maintaining marginal and joint laws.

Transformer-based autoregressive models \citep{vaswani2017attention} provide a natural framework for sequential data generation. However, their empirical success has been largely confined to {\it discrete modalities}, most notably next-token generation in natural language, and extending them to continuous settings remains challenging and fragile \citep{esser2021taming, janner2021offline}. Existing approaches often rely on discretization or quantization, which can introduce non-negligible approximation error in continuous spaces \citep{li2022adacat,lee2022autoregressive,li2024autoregressive}. In contrast, diffusion models have emerged as a flexible and powerful paradigm for modeling and sampling from continuous distributions in high-dimensional spaces \citep{song2019generative,song2020score,ho2020denoising}, with strong empirical performance and increasing theoretical support \citep{dhariwal2021diffusion,rombach2022high,betker2023improving,watson2023novo,hoogeboom2022equivariant}. Despite these advances, most diffusion-based methods are designed for static data or settings where the entire sample is treated as a single object. A direct extension to sequential data—jointly noising and denoising the full trajectory—can match the joint distribution, but does not enforce the underlying information flow. In particular, it does not guarantee that each generated value depends only on past information, and thus may violate the causal ordering inherent in sequential systems.

This gap motivates the central question of our paper:
\vspace{3pt}
\begin{center}
\centering
\emph{How can we unlock diffusion models for sequential data generation, while preserving the adaptiveness and retaining theoretical guarantees?}
\end{center}
\vspace{3pt}
We address this question by introducing a sequential forward--backward diffusion framework, Adaptive Diffusion for Sequential Data Generation \mbox{(\method)}, that aligns the diffusion dynamics with the temporal information flow. The key idea is to generate the sample path through conditional reverse diffusion steps, such that each generated coordinate depends only on the previously generated history. This construction preserves the adapted structure of sequential data while retaining the flexibility of diffusion-based generation.

{ 
Beyond full time series generation, this conditional perspective also suggests a broader use of \method{} as a generator of admissible scenarios conditional on available information. Given a realized history, the model can generate multiple future continuations that respect the same information flow, which can then be used in downstream tasks such as statistical inference, predictive control, and risk evaluation. Although we focus on time series in this paper, the same principle can be applied whenever generated data must respect certain causal structure, including panel data, network data, and causal inference.
}

\medskip

\paragraph{Our work and contributions.} Our contributions are threefold.

{\bf (i) A new diffusion model framework for adapted sequential data generation.}
We introduce a novel Adaptive Diffusion for Sequential Data Generation \mbox{(\method)~framework} for generating time series that are adapted to an information flow. Unlike classical diffusion models that operate on the full time series as a single object, our construction generates each coordinate sequentially, conditioning only on the previously generated history. This yields a generative model that preserves the conditional distributions of the data while enforcing adaptiveness by design. At a conceptual level, this provides a principled way to incorporate filtration constraints into diffusion models. On the algorithmic side, we develop a novel score-matching objective that decouples the learning problem across time steps, enabling parallel training of the score functions and ensuring scalability to high-dimensional sequences. Beyond synthetic data generation, this framework opens the door to a broad range of downstream applications, including multi-step prediction, predictive decision-making, and statistical inference under information-flow constraints.

{\bf (ii) Statistical learning theory for the \method~framework.}
We develop a general statistical learning theory for our proposed \mbox{\method{}~framework}, including  score approximation, score estimation, and distribution estimation guarantees. Our analysis is network architecture-agnostic and can incorporate generic approximation and covering number assumptions. 
A key technical challenge is to mitigate the growing length of the history for long-horizon generation. We introduce a history truncation mechanism based on the temporal dependence decay in the sequence. This yields the following finite-sample total-variation bound for the conditional law generated by \method{}:
$$\tilde{\cO}(n^{-\frac{\beta}{4(d_{\rm trunc} + \beta)}}).$$ 
Here, $n$ is the sample size and $\beta$ is the distribution smoothness parameter defined in Assumption~\ref{ass:holder-re}. The rate depends on the effective history length $d_{\rm trunc}$, rather than the full horizon length, showing that long-horizon generation can remain statistically tractable when temporal dependence decays; see for example the exponential and polynomial decays in Propositions \ref{cor:exp-decay}--\ref{cor:poly-decay}. Moreover, despite the additional adaptiveness constraint, this rate matches the best-known rates for static diffusion models.

{\bf (iii) Empirical validation for temporal structure learning and decision-making.}
We implement the proposed framework using transformer architectures with causal masking and evaluate it on both synthetic and real data. The empirical results indicate that the method captures meaningful aspects of the sequential dependence structure, rather than merely reproducing static marginals. On ARMA models, it accurately recovers the autocorrelation structure, including in small-sample regimes, and produces synthetic samples whose estimated autocorrelation functions are more accurate than those obtained directly from the empirical training data. On Gaussian processes, it reconstructs the covariance matrix with the correct decay pattern, showing that the model captures the dependence structure of the underlying process. We further evaluate performance in a downstream decision problem,  namely mean--variance portfolio optimization on the S\&P 500. In this experiment, policies trained using diffusion-generated samples attain the highest Sharpe ratios among several benchmark methods. These findings suggest that respecting temporal dependence and information flow in sequential data generation can lead to measurable improvements in both statistical fidelity and downstream decision quality.

\vspace{5pt}
\paragraph{Related literature.}
Our work is related to several emerging research directions.

First, a variety of generative modeling frameworks have been developed for time series synthesis. Prior to the emergence of diffusion models, most approaches were based on the GAN framework \citep{mirza2014conditional,esteban2017real,fu2019time,yoon2019time,koshiyama2020generative,ni2020conditional,li2020generating,vuletic2023fin,volgan,wiese2020quant,cont2025tail,Lou2024PGg} while others adopted the VAE paradigm \citep{buehler2020generating,Buehler2020Add,desai2021timevae,Wiese2021MAS,cai2023hybrid,liu2022time,huang2024generative,acciaio2024time}. Diffusion models have gained recent attention for  time series generation \citep{rasul2021autoregressivedenoisingdiffusionmodels,lim2023regulartimeseriesgenerationusing,lim2024tsgm,yuan2024diffusionts,naiman2024utilizingimagetransformsdiffusion,yang2024survey, aghapour2025solving}. In parallel, methods based on Schrödinger Bridges have also been explored for generative modeling of time series \citep{wang2021deepgenerativelearningschrodinger,Debortolietal21,hamdouche2023generativemodelingtimeseries,alouadi2025robust}.

Our work is also closely related to the recent advances of diffusion models.  A line of research focused on the convergence theory of diffusion models \citep{chen2022sampling,lee2022convergence,liu2022let,lee2023convergence,chen2023improved,benton2024nearly,chen2023restoration,li2023towards,pedrotti2023improved,cheng2023convergence,huang2024convergence,liang2024non,li2024d,tang2024contractive,li2024unified,ren2024discrete,gao2024convergence,gentiloni2025beyond,hu2024statistical,liang2025low,gottwald2025localized,mei2023deep,zhang2026diffusion,li2024provable,cai2025minimax,wu2024theoretical,li2025dimension,wu2024stochastic}. In parallel, several studies analyzed the statistical properties of score matching, establishing statistical error rates and sample complexity bounds \citep{oko2023diffusion,wibisono2024optimal,zhang2024minimax,han2024neural,dou2024optimal,chen2023score,fu2024unveil,hu2024statistical,zhang2026diffusion}. However, none of the prior work addressed the statistical estimation theory of diffusion models for sequential data, except \cite{fu2024diffusion}, which investigated diffusion transformers restricted to (discrete) Gaussian process data.

Compared with the existing literature, our work is the first to study diffusion models for generic sequential data with rigorous statistical guarantees. More fundamentally, the proposed framework is designed to preserve not only marginal or joint path distributions, but also the temporal dependence and information structure inherent in the underlying data-generating process. In particular, the adaptiveness guarantee established in this work makes diffusion models applicable to settings where respecting the information flow is essential, and thereby extends their scope beyond data generation to tasks such as multi-step prediction, predictive decision-making, and statistical inference under information flow constraints.

 \vspace{5pt}
\paragraph{Notations.}
Let $(\Omega, \cF, \P, (\cF_h)_{h=1}^H)$ be a filtered probability space. Denote $(W_t^h)_{h=1}^H$ and $(\bar{W}_t^h)_{h=1}^H$ by $2H$ independent one-dimensional Brownian motions. For a vector $\bX = (X^1, \dots, X^H) \in \R^H$, we denote $\bX^{[\ell:h]} \coloneqq (X^\ell, \dots, X^h)$ and $\bX^{[\ell:h)} \coloneqq (X^\ell, \dots, X^{h-1})$. Given two random variables $X$ and $Y$, we denote $\E_X[\cdot]$ as the expectation over $X$ and denote $\E_{X\vert Y}[\cdot]$ as the conditional expectation of $X$ given $Y$. Let $\cN(\mu, \sigma^2)$ be the one-dimensional Gaussian distribution with mean $\mu$ and variance $\sigma^2$. We use the notation $X \sim p$ to mean random variable $X$ with probability density $p$.  Let $\sigma(\bX)$ be the $\sigma$-algebra generated by the random vector $\bX$. For two $\sigma$-algebras $\mathcal{A}$ and $\mathcal{B}$, $\mathcal{A} \vee \mathcal{B}$ denotes the smallest $\sigma$-algebra containing $\mathcal{A} \cup \mathcal{B}$. Finally, $\norm{\cdot}$ is the Euclidean norm. We use the notation $X\stackrel{d}{=}Y$ to mean that $X$ and $Y$ have the same distribution. Define the set $[p] = \curly{1, \dots, p}$ for $p \in \mathbb{Z}_+$.

\section{Problem Setup}\label{sec:model}

Our objective is to design a diffusion model that generates adapted time series whose law matches the underlying data distribution%
, with respect to a specified (and meaningful) filtration. For notational convenience, we present the development in the one-dimensional case: let $\bX_0 = (X_0^h)_{h=1}^H \in \mathbb{R}^H$ denote a length-$H$ time series in discrete-time and adapted to the filtration $(\cF_h)_{h=1}^H$, i.e.,  $X_0^h\in \mathcal{F}_h$. The methodology extends straightforwardly to higher dimensions. In addition, denote $\mu_0^{[1:H]}\in \mathcal{P}(\mathbb{R}^H)$ as the distribution of $\bX_0$.

\paragraph{Preliminaries on diffusion models.} In the classical setting \citep{song2019generative,song2020score,ho2020denoising}, diffusion models include a forward process that gradually adds noise to data and a time-reverse process that gradually removes noise to generate new data. Mathematically, the forward process is often chosen to be an $H$-dimensional Ornstein-Uhlenbeck (OU) process defined on $\R^H$:
\begin{equation}\label{eq:diffusion}
\dd \bX'_t = - \frac{1}{2}g(t)\bX'_t\dd t + \sqrt{g(t)}\dd {\bf W}_t', \quad \bX_0' \sim \mu_0^{[1:H]},   
\end{equation}
where $\{{\bf W}'_t\}_{0 \leq t\leq T}$ is a $H$-dimensional Brownian motion and $g: \R \to \R_+$ is a positive-valued function. The time-reversed SDE of \eqref{eq:diffusion} is given by
\begin{equation}
\label{eq:reversed_diff}
    \dd\overline{\bf X}'_t = \Big[\frac{1}{2}g(T-t)\overline{ \bX}'_t + g(T-t)\nabla_x \log p_{T-t} (\overline{\bf X}_t')\Big]\dd t + \sqrt{g(T-t)}\dd \overline{\bf W}'_t,  
\end{equation}
where  $\{\overline{\bf W}'_t\}_{0 \leq t\leq T}$ is another independent $H$-dimensional Brownian motion and $p_t: \R^H \to \R$ represents the marginal density  of $\bX_t'$ at time $t$ in the forward process \eqref{eq:diffusion}.

Such a framework is appropriate for generating static data or fully predictable sequences (e.g., videos), where no filtration-based adaptiveness constraint is imposed. Even when the forward OU process is driven by independent Gaussian noise across coordinates, the corresponding backward dynamics generally fail to preserve the information flow of the data, since the score term $\nabla \log p_{T-t}(\cdot)$ depends on the joint density and thus introduces cross-coordinate dependence. As a result, the value at a given time step may implicitly depend on future components of the sequence. Consequently, adaptiveness with respect to a prescribed filtration is not enforced, and the generated data may violate the intended information flow or temporal structure. This limitation motivates us to develop the \mbox{\method{}~framework}, which adds and removes noise progressively along the time index, conditioning only on previously generated history, and thereby preserves the prescribed information flow by construction.

\paragraph{Adaptive Sequential Diffusion Model.}  For each coordinate/timestamp $h$ $(1\leq h \leq H)$, we  define a pair of forward and time-reversed one-dimensional SDEs, driven by OU processes. In the construction below, we specialize to the constant noise schedule $g(\cdot)\equiv 1$.

For $h = 1$, we consider the following one-dimensional forward process on $\R$:
\begin{align}\label{eq:forward-1}
\dd X_t^1 = -\tfrac{1}{2}X_t^1 \dd t + \dd W_t^1, \;\; X_0^1 \sim \phi^1(0, \cdot),
\end{align} 
where $\phi^1(0, \cdot): \R \to \R$ is the probability density of  $X_0^1$.  To generate new samples, we consider another stochastic process on $\R$ for sampling associated with
\eqref{eq:forward-1}:
\begin{align}\label{eq:backward-1}
\dd \bar{X}_t^1 =  \Big[\tfrac{1}{2}\bar{X}_t^1 + \partial_x\log  \phi^1(T-t, \bar{X}_t^1)\Big]\dd t + \dd \bar{W}_t^1, \;\; \bar{X}_0^1 \sim \phi^1 (T, \cdot),
\end{align}
where $\phi^1(t, \cdot): \R \to \R$ is the probability density of $X_t^1$ and $\partial_x\log  \phi^1(t, \cdot): \R \to \R$ denotes its score function. Under regularity conditions, the forward process \eqref{eq:forward-1} admits a strong solution and thus the score function is well-defined. For each $h > 1$, we construct a forward process on $\R$ {\it iteratively} by conditioning on the~history:
\begin{align}\label{eq:forward-h}    \dd X_t^h = -\frac{1}{2}X_t^h \dd t + \dd W_t^h, \;\; X_0^h \sim \phi_{\bar{\bX}_T^{[1:h)}}^h(0, \cdot),
\end{align}
where for any \(\bz\in\R^{h-1}\) and \(t\in[0,T]\), \(\phi_{\bz}^h(t,\cdot):\R\to\R\) denotes the conditional density function of \(X^h_{t}|\bX^{[1:h)}_0=\bz\). 
Compared to \eqref{eq:forward-1}, the forward process \eqref{eq:forward-h} is initialized from a conditional probability distribution,  incorporating the temporal dependence inherent in the time series. Similar to the sampling process \eqref{eq:backward-1}, we define another stochastic process on $\R$ for sampling:
\begin{align}\label{eq:backward-h}
\textstyle        \dd \bar{X}_t^h =  \Big[\frac{1}{2}\bar{X}_t^h + \partial_x \log \phi_{\bar{\bX}_T^{[1:h)}}^h(T-t, \bar{X}_t^h)\Big]\dd t + \dd \bar{W}_t^h, \;\; \bar{X}_0^h \sim \phi_{\bar{\bX}_T^{[1:h)}}^h(T, \cdot),
\end{align}
with \(\partial_x\log{\phi^h_{\bar{\bX}_{T}^{[1:h)}}}(T-t,\bar X^h_t)=\partial_x\log{\phi^h_{\bz}(T-t,x)}\biggl|_{\bz=\bar{\bX}^{[1:h)}_T,\;x=\bar X^h_t}\).
Here the sampling process \eqref{eq:backward-h} incorporates previously generated trajectories $\bar{\bX}_T^{[1:h)}$, ensuring that the synthesized series $\bar{\bX}_T = (\bar{X}_T^h)_{h=1}^H$ is generated sequentially. 
{ The following proposition verifies the key distributional consistency property of this construction: with the exact conditional score and the corresponding terminal marginal initialization, the sampling processes \eqref{eq:backward-1} and \eqref{eq:backward-h} recover the desired marginal and conditional laws.}

\begin{proposition}[Property of the sampling processes]\label{prop:time reversal}
    Let $(\bar{X}_T^h)_{h=1}^H$ follow \eqref{eq:forward-1}-\eqref{eq:backward-h}. Then the following distributional property holds:
    \begin{align*}
        \bar{X}_T^1 \stackrel{d}{=} X_0^1, \;\; \text{and} \;\; (\bar{X}_T^h\vert \bar{\bX}_T^{[1:h)}=\bx)\stackrel{d}{=} (X_0^h \vert \bX_0^{[1:h)}=\bx),\ \ \forall h>1,\;\bx \in \R^{h-1}.
    \end{align*}
    Consequently, we have $(\bar{X}_T^h)_{h=1}^H \stackrel{d}{=} (X_0^h)_{h=1}^H$.
\end{proposition}
 { This law-matching property justifies referring to the sampling processes \eqref{eq:backward-1} and \eqref{eq:backward-h} as the time-reversed SDEs associated with the conditional forward diffusions.}

\begin{proof}
We only prove the property when $h > 1$ as the case where $h = 1$ follows directly from \cite{anderson1982reverse,haussmann1986time}. 
From the forward process \eqref{eq:forward-h}, we know that $\phi_\bz^h(t, \cdot)$ satisfies the Fokker-Planck equation (in the weak sense):
\begin{align}\label{eq:fk-transition}
    \frac{\partial}{\partial t}\phi_\bz^h(t, x) = \frac{1}{2}\partial_{xx} \phi_\bz^h(t, x) +  \frac{1}{2}\partial_x \Big(\phi_\bz^h(t, x) x \Big),
\end{align}
where we recall that $\phi^h_\bz(0, \cdot)$ is the probability density of $X_0^h$ given the history $\bX_0^{[1:h)} = \bz$.

Denote $\psi_\bz^h(t, x) \coloneqq \phi_\bz^h(T-t, x)$ by the time reversed probability density. Differentiate $\psi_\bz^h(t, x)$ w.r.t. $t$ to have
\begin{align}\label{eq:FK-reversal}
    \frac{\partial}{\partial t}\psi_\bz^h(t, x) & = - \frac{\partial}{\partial s} \phi_\bz^h(s, x)\vert_{s = T - t} \nonumber \\ & = -\frac{1}{2}\partial_{xx} \phi_\bz^h(T-t, x) - \frac{1}{2}\partial_x \Big( \phi_\bz^h(T-t, x) x \Big) \nonumber \\ & = -\frac{1}{2}\partial_{xx} \psi_\bz^h(t, x) - \frac{1}{2}\partial_x \Big( \psi_\bz^h(t, x) x \Big).
\end{align} 
Define the function $\bar{b}^h: [0, T] \times \R \times \R^{h-1} \to \R$ as
\begin{align*}
    \bar{b}^{h}(t, x, \bz) \coloneqq \frac{\partial_x \psi_\bz^h(t, x)}{\psi_\bz^h(t, x)} + \frac{1}{2}x= \partial_x \log \psi_\bz^h(t, x) + \frac{1}{2}x.
\end{align*}
Note that
\begin{align*}
    \partial_{xx} \psi_\bz^h(t, x) + \frac{1}{2}\partial_x\Big( \psi_\bz^h(t, x)x\Big) & = 
    \partial_x \Big(\partial_x\psi_\bz^h(t, x)       + \frac{1}{2}\psi_\bz^h(t, x)x\Big)\\ & = \partial_x \parenthesis{\psi_\bz^h(t, x) \Big(\partial_x \log \psi_\bz^h(t, x) + \frac{1}{2}x\Big)}   \\ & = \partial_x \Big(\psi_\bz^h(t, x)  \bar{b}^{h}(t, x, \bz)\Big)
\end{align*}
As a consequence, it follows from \eqref{eq:FK-reversal} that
\begin{align}\label{eq:backward-FP}
    \frac{\partial}{\partial t}\psi_\bz^h(t, x) = \frac{1}{2}\partial_{xx} \psi_\bz^h(t, x) -   \partial_x \Big(\psi_\bz^h(t, x)  \bar{b}^h(t, x, \bz)\Big).
\end{align}
This is exactly the Fokker-Planck equation of the SDE \eqref{eq:backward-h}.%

\end{proof}
\section{Algorithm design and adaptiveness property}

In practice, the time-reversed SDEs \eqref{eq:backward-1} and \eqref{eq:backward-h} cannot be directly used to generate samples,  as both the score functions $\partial_x \log \phi_{\bar{\bX}_T^{[1:h)}}^h(t,\cdot)$ and the initial distributions $\phi_{\bar{\bX}_T^{[1:h)}}^h(T, \cdot)$ are unknown. To address this issue, we follow common
practice to replace the initial distribution by the standard Gaussian distribution and replace the ground-truth score $\partial_x \log \phi^h_\bz(t, x)$ by a score estimator $\hat{s}^h(t, x, \bz): [0, T] \times \R \times \R^{h-1} \to \R$. By convention, we set $\bz = \varnothing$ when $h = 1$ and thus $\hat{s}^1(t, x, \varnothing) = \hat{s}^1(t, x)$.  With these modifications, we obtain an
approximation of the time-reversed process, which is practically~implementable:\vspace{-5pt}
 \begin{align}\label{eq:backward-approx}
    \dd \tilde{X}_t^h =  \Big[\tilde{X}_t^h/2 + \hat{s}^h(T-t, \tilde{X}_t^h, \tilde{X}_T^{[1:h)})\Big]\dd t + \dd \bar{W}_t^h, \;\; \tilde{X}_0^h \sim \cN(0, 1).\vspace{-5pt}
\end{align}
  Using \eqref{eq:backward-approx}, we propose Algorithm \ref{algo:sample-generation} for sequential data sampling.\footnote{The displayed score input uses the full generated history for notational simplicity. The same design permits replacing $\bx^{[1:h)}$ by a truncated recent-history window $\bx^{[k:h)}$ in $\hat{s}^h$, as in the empirical loss below, without changing the sequential sampling nature.}
For ease of exposition, Algorithm~\ref{algo:sample-generation} is written in continuous time for simplicity while implemented with a time discretization.

\begin{algorithm}[t]
    \caption{Sampling algorithm for sequential data}
    \begin{algorithmic}[1] \label{algo:sample-generation}
    \STATE {\bf {Input:}} Length of time series $H$, early-stop time $t_0$, terminal time $T$, and score networks~$\curly{\hat{s}^h(\cdot)}_{h=1}^H$.
    \STATE Run the backward process up to $T-t_0$ 
    \begin{align*}
        \dd \tilde{X}_t^1 = \left[\tilde{X}^1_t/2 + \hat{s}^1(T-t, \tilde{X}_t^1)\right] \dd t + \dd \bar{W}^1_t, \text{ with } \tilde{X}^1_0\sim \cN(0,1),
    \end{align*}
    and obtain a realization $\tilde{X}^1_{T-t_0} = x^{1}$. 
    \FOR{$h = 2, \dots, H$}
    \STATE  Run the backward process up to $T-t_0$
    \begin{align*}
        \dd \tilde{X}_t^h = \left[\tilde{X}^h_t/2 + \hat{s}^h(T-t, \tilde{X}_t^h, \bx^{[1:h)})\right] \dd t + \dd \bar{W}^h_t, \text{ with } \tilde{X}^h_0\sim \cN(0, 1),
    \end{align*}
    and obtain a realization $\tilde{X}_{T-t_0}^h = x^{h}$.
    \STATE Set $\bx^{[1:h]} = ((\bx^{[1:h)})^\top, x^h)^\top$ as a realization of $\tilde{\bX}^{[1:h]}_{T-t_0}$. 
    \ENDFOR
    \RETURN $\bx^{[1:H]}$ as a sample path of $\tilde{\bX}_{T-t_0}^{[1:H]}$.
    \end{algorithmic}
\end{algorithm}

\paragraph{Score matching.}
To estimate the score function, a natural choice is to minimize the weighted quadratic loss for all coordinates: %
\begin{align}\label{eq:score-matching-naive} 
\min_{s^h\in \cS^h}\;\int_{t_0}^T \frac{1}{T- t_0}\E_{\bar{\bX}^{[1:h)}_T}\E_{X_t^h \vert \bar{\bX}^{[1:h)}_T}\big[\big|s^h(t, X_t^h, \bar{\bX}^{[1:h)}_T) -  \partial_x \log \phi^h_{\bar{\bX}^{[1:h)}_T}(t, X_t^h)\big|^2\big]\dd t,
\end{align}
where $\cS^h$ is a function class (often neural networks) for coordinate $h$. Here, $t_0$ is an early-stopping time to
prevent the blow-up of score functions, which is commonly adopted in practice \citep{song2020score,chen2023score,nichol2021improved}. As $\bar{\bX}_T^{[1:h)}$ has the same distribution as $\bX_0^{[1:h)}$ for each coordinate $h$, one can equivalently minimize the alternative:
\begin{align}\label{eq:score-matching-equiv}
\min_{s^h\in \cS^h}\;\int_{t_0}^T \frac{1}{T- t_0}\E_{\bX_0^{[1:h)}}\E_{X_t^h \vert \bX_0^{[1:h)}}\big[\big|s^h(t, X_t^h, \bX_0^{[1:h)}) -  \partial_x \log \phi^h_{\bX_0^{[1:h)}}(t, X_t^h)\big|^2\big]\dd t.
\end{align}
As shown by \citet{vincent2011connection}, rather than minimizing the integral above, we can minimize an equivalent denoising score matching objective which shares the same minimizer as \eqref{eq:score-matching-equiv}: %
\begin{align}\label{eq:dsm-naive}
\min_{s^h\in \cS^h}\;\int_{t_0}^T \frac{1}{T- t_0}\E_{\bX_0^{[1:h)}}\E_{X_0^h \vert \bX_0^{[1:h)}}\E_{X_t^h \vert X_0^h}\big[\big|s^h(t, X_t^h, \bX_0^{[1:h)}) -  \sigma_t^{-2}(\alpha_t X_0^h - X_t^h)\big|^2\big]\dd t,
\end{align}
with $\alpha_t = e^{-t/2}$ and $\sigma_t^2 = 1 - e^{-t}$. Furthermore, we denote $\bs = (s^1, \dots, s^H)$, where the score $s^h$ takes current state $X_t^h$ and the history $\bX_0^{[1:h)}$ as input. Using this vector notation, we rewrite \eqref{eq:dsm-naive} as
\begin{align}\label{eq:dsm-final}
\min_{\bs\in \cS}\;\int_{t_0}^T \frac{1}{T- t_0}\E_{\bX_0}\E_{\bX_t\vert \bX_0}\Big[\Big\|\bs(t, \bX_t, \bX_0) -  \sigma_t^{-2}(\alpha_t \bX_0 - \bX_t)\Big\|^2\Big]\dd t.
\end{align}
Here, the function class $\cS$ is an aggregation of all function classes at each coordinate $h$, i.e., $\cS =  \cS^1 \times \cS^2 \times \cdots \times\cS^H$. We remark that the objective \eqref{eq:dsm-final} allows for training score functions in parallel in contrast to the sequential sampling in Algorithm \ref{algo:sample-generation}. In fact, \eqref{eq:dsm-final}  can be implemented in a {\it parallel fashion} for each timestamp $h\in[H]$, without sequentially generating prior timestamps during training.

In practice, we approximate \eqref{eq:dsm-final} by its empirical version using available data points. This full-history objective becomes statistically less efficient for later coordinates: as $h$ increases, the input dimension of $s^h$ grows with the conditioning history.
In many financial applications, however, dependence on remote coordinates decays with lag \citep{lo1991long,cont2001empirical}. We therefore condition the score on a recent-history window.
Specifically, given $n$~i.i.d.~sample trajectories $\cD \coloneqq \curly{\bx_i}_{i = 1}^n$ with $ \bx_i\in \R^H$, we sample $X_t^h$ given $X_0^h = x_i^h$ from $\cN(\alpha_t x_i^h, \sigma_t^2)$. We focus on the truncation window  from coordinate $k$ to $h$ and denote the associated loss function by %
\begin{align} \label{eq:loss-lh}
    \ell^{[k:h]}(\bx_i^{[k:h]}; s) = \int_{t_0}^T \frac{1}{T - t_0}\E_{X_t^h\sim \cN(\alpha_t x_i^h, \sigma_t^2)}\big[\big|s(t, X_t^h, \bx_i^{[k:h)}) -  \sigma_t^{-2}(\alpha_t x_i^h - X_t^h)\big|^2\big]\dd t.
\end{align}
Here the truncation window can be properly chosen to balance the estimation efficiency and the loss of historical information; see Section \ref{sec:relu}.
Let $\hat{s}^h$ be a minimizer to the empirical score matching risk $\widehat{\cL}^{[k:h]}(s) := \frac{1}{n}\sum_{i = 1}^n \ell^{[k:h]}(\bx_i^{[k:h]}; s)$.

\paragraph{Adaptiveness property.} 
{ 
We now establish the adaptiveness property of the time series generated by Algorithm~\ref{algo:sample-generation}. This property is {\it not automatic} for vanilla diffusion models. Specifically, when a vanilla diffusion model is applied to the full time series jointly, the reverse-time score \(\nabla \log p_t(x^1,\ldots,x^H)\) generally depends on all time indices, so the update of an earlier coordinate may use information from later coordinates. Such dependence breaks the intended information structure of a sequential system. By contrast, Algorithm~\ref{algo:sample-generation} generates coordinates sequentially and uses only the previously generated history when sampling the next coordinate.
}

\begin{assumption}\label{ass:score-learning}
    For each coordinate $h$, the score estimator $\hat{s}^h$ is trained such that
        $\hat{s}^h\in \mathcal{F}_h\vee \sigma(\tilde{U}^h)$,
    in which  $\tilde{U}^h$ is an independent random variable defined on the same probability space. %
\end{assumption}
Here, $\tilde{U}^h$ can be viewed as exogenous randomness occurring in training ({\it e.g.,} adopting a stochastic optimization algorithm). 
Assume the probability space \((\Omega, \mathcal F,\mathbb P)\) is rich enough to define Brownian motions $\bar W_t^h$, random variables $\tilde{U}^h$, and standard Gaussian random variables $\tilde{X}^h_0\sim \cN(0, 1)$, \(h\in\{1,\dots,H\}\). The next proposition constructs a filtration such that the generated sequential data is~adapted.

\begin{proposition}[Adaptiveness] \label{prop:scalar-learned-model}
Suppose Assumption \ref{ass:score-learning} holds. In addition, assume $\bar W_t^h$, $\tilde{U}^h$, and $\tilde{X}^h_0\sim \cN(0, 1)$, \(h\in\{1,\dots,H\}\) are independent.\footnote{This independence assumption simplifies the statement; our construction accommodates dependence across coordinates.}  Then the output $ (\tilde{X}_T^h)_{h=1}^H$ from Algorithm \ref{algo:sample-generation} is adapted to the enlarged filtration $\tilde{\mathcal{H}}\coloneqq(\tilde{\mathcal{H}}_h)_{h=1}^H$ defined as, for $1\leq h \leq H$
    \begin{eqnarray*}
       \tilde{\mathcal{H}}_1 &=& \sigma(\tilde{X}^1_0) \vee \sigma(\{\bar{W}_t^1\}_{t=0}^T)\vee \sigma(\tilde{U}^1)\vee \mathcal{F}_1,\\ \tilde{\mathcal{H}}_{h} &=& \tilde{\mathcal{H}}_{h-1} \vee \sigma(\tilde{X}^h_0) \vee   \sigma(\{\bar{W}_t^h\}_{t=0}^T)  \vee \sigma(\tilde{U}^h)\vee \mathcal{F}_h.
    \end{eqnarray*}
\end{proposition}

To interpret Proposition \ref{prop:scalar-learned-model}, recall that the enlarged filtration is obtained by adjoining to the data filtration the auxiliary randomness used in training and sampling. Concretely, it is the filtration generated by
(i) the data filtration $(\cF_h)_{h\ge 0}$,
(ii) the training randomization $\sigma(\tilde U^h)$,
(iii) the (random) initialization $\sigma(\tilde X_0^h)$, and
(iv) the driving Brownian motion $\sigma(\bar W^h_t:0\le t\le T)$.
Equivalently, $\tilde{\cH}$ is the smallest $\sigma$-field (and the associated smallest filtration, upon taking the natural time-indexed version) with respect to which all random objects appearing in the training procedure and in the sampling dynamics are measurable.

\section{Distribution Estimation Theory} \label{sec:main results}
In this section, we present our main distribution estimation results for \method{}. To ease the presentation, we consider utilizing the continuous-time
backward processes \eqref{eq:backward-1} and \eqref{eq:backward-h} for distribution estimation. In practice, a proper discretization is applied
to generate samples, whose deviation from the continuous-time backward process can be controlled
by the step size of the discretization; see \citet[Theorem 2]{chen2022sampling}, which requires a Lipschitz score, and \citet{benton2024nearly}, which removes this assumption via stochastic localization.

The key step is decomposing joint distribution learning to a series of conditional distribution learning given the  sequential conditional nature of our framework. For learning one-step conditional distributions, we leverage recent advances in conditional diffusion models \citep{fu2024unveil,hu2024statistical,tang2024conditional}. Recall the ground-truth joint distribution is $\mu_0^{[1:H]}$ and we denote the learned joint distribution of $\tilde{\bX}_{T-t_0}^{[1:H]}$ in Algorithm \ref{algo:sample-generation} by $\hat{\mu}_{t_0}^{[1:H]}$. In addition,  let $\mu^h_0[\bz]$ be the ground-truth distribution of $X^h_0 \vert \bX_0^{[1:h)} = \bz$ and $\hat{\mu}_{t_0}^h[\bz]$ be the learned distribution of $\tilde{X}_{T-t_0}^h \vert \tilde{\bX}_{T-t_0}^{[1:h)} = \bz$.  The chain rule of TV distance implies that
\begin{align}
    {\rm TV}(\mu_0^{[1:H]}, \hat{\mu}_{t_0}^{[1:H]}) \leq \sum_{h = 1}^H  \E_{\bX_0^{[1:h)} \sim \mu_0^{[1:h)}}\bracket{{\rm TV}\parenthesis{\mu_0^{h}[\bX_0^{[1:h)}], \hat{\mu}_{t_0}^{h}[\bX_0^{[1:h)}]}}. \label{eq:tv-chain}
\end{align}
Inspired by \eqref{eq:tv-chain}, it suffices to bound the one-step conditional distribution for each $h \in [H]$ in TV distance. The truncated nature of loss function $\ell^{[k:h]}(\cdot)$ implies $\hat{\mu}^h_{t_0}[\bz^{[1:h)}] = \hat{\mu}^h_{t_0}[\bz^{[k:h)}]$ for any $\bz^{[1:h)} \in \R^{h-1} $. %
Moreover, triangle inequality implies
\begin{align}
    \E_{\bX_0^{[1:h)}}\bracket{{\rm TV}(\mu_0^{h}[\bX_0^{[1:h)}], \hat{\mu}_{t_0}^{h}[\bX_0^{[1:h)}])} & \leq \underbrace{\E_{\bX_0^{[k:h)}}\bracket{{\rm TV}(\mu_0^{h}[\bX_0^{[k:h)}], \hat{\mu}_{t_0}^{h}[\bX_0^{[k:h)}])}}_{\text{learning error}} \nonumber \\ & \qquad + \underbrace{\E_{\bX_0^{[1:h)}}\bracket{{\rm TV}(\mu_0^{h}[\bX_0^{[1:h)}], \mu_0^{h}[\bX_0^{[k:h)}])}}_{\text{truncation error}}. \label{eq:tv-triangle}
\end{align}
Here, we overload the notation $\mu^h_0[\bz^{[k:h)}]$ as the distribution of $X_0^h \vert \bX_0^{[k:h)} = \bz^{[k:h)}$ for any $\bz \in \R^{h - k}$. 
We make the following assumption on the truncation error. 

\begin{assumption}\label{assumption:dependence_decay}
Denote $d_{\rm trunc} = h-k+1$ as the truncated sequence length. Fix $h \in [H]$ and consider $k < h$, it holds that 
\begin{align*}
\E_{\bX_0^{[1:h)}}\bracket{{\rm TV}(\mu_0^{h}[\bX_0^{[1:h)}], \mu_0^{h}[\bX_0^{[k:h)}])} \leq u(d_{\rm trunc}),
\end{align*}
where $u(\cdot)$ only depends on  $d_{\rm trunc}$ and is monotonically decreasing.
\end{assumption}
This assumption imposes a quantitative decay-of-dependence property in expectation: the expected total variation distance between the conditional law given the full history and that given only the most recent $d_{\rm trunc}$ observations is bounded by $u(d_{\rm trunc})$. As $u$ decreases in $d_{\rm trunc}$, the dependence on the remote past vanishes progressively. Such an assumption is justified in many financial applications \citep{lo1991long,cont2001empirical}. %
In particular, our numerical experiments in Section \ref{sec:mean-var} show that log returns of the S\&P 500 index are close to Markovian \citep{martin2021autocorrelation, cont2001empirical}.
Thus, conditioning on a short history is nearly as informative as conditioning on the full past. The decay rate $u(\cdot)$ plays an important role in our theoretical analysis.  We will explicitly discuss how different decay regimes influence the choice of $k$ in Section \ref{sec:relu}.

It remains to bound the learning-error term in \eqref{eq:tv-triangle}. 
{  To proceed, we decompose it informally into three sources: the bias from early stopping, the mismatch in the initial distribution, and the error from score estimation. The first source comes from stopping the time-reverse SDE at time $T-t_0$ rather than evolving it all the way to time $T$. The second arises because the sampling process is initialized from the standard normal distribution instead of the exact forward marginal at time $T$. The third is due to replacing the true conditional score along the sampling trajectory by its learned approximation $\hat{s}^h$. This decomposition is discussed in more detail in Section~\ref{sec:dist-est}.

We then reduce the score-estimation contribution to a statistical learning problem, consisting of an approximation component and a finite-sample statistical error component. Section~\ref{sec:score-approx} controls the former by showing that the neural network class contains a candidate whose discrepancy from the target conditional score is bounded at a prescribed rate. Section~\ref{sec:score-est} controls the latter by proving that the empirical score-matching minimizer attains small population score-matching risk. Combining these two estimates yields a bound on the score-estimation contribution, and hence on the learning error and the overall distribution estimation error.}

\subsection{Score Approximation}\label{sec:score-approx}
We devote this subsection to conditional score approximation theory.  For ease of exposition, we overload $\phi^h_\bz(0, \cdot)$ as the conditional density of $X_0^h$ given $\bX_0^{[k:h)} = \bz$ when the conditional information is clear from the context. %
We make the following assumption on each conditional distribution~$X_0^h \vert \bX_0^{[k:h)}$.

\begin{assumption}[Conditional sub-Gaussianity]\label{ass:truncated-subGaussian}
    There are constants $C_1^1,C_2^1>0$ such that the marginal density $\phi^1(0,\cdot)$ satisfies
    \begin{align*}
        \phi^1(0,x) \leq C_1^1\exp\parenthesis{-C_2^1\abs{x}^2/2}
    \end{align*}
    for all $x\in\R$. For each $h \in [H]$ and $k\in [h-1]$, there exist two positive constants $C_1^h, C_2^h > 0$ and a vector $\bv^h \in \R^{h - k}$ such that
    \begin{align*}
        \phi_{\bz}^h(0,x) \leq C_1^h\exp\parenthesis{-C_2^h \abs{x - (\bv^h)^\top \bz}^2/2}
    \end{align*}
    for all $x \in \R$ and $\bz \in \R^{h - k}$.
\end{assumption}
Throughout the analysis, we omit the dependence of $C_1, C_2$, and $\bv$ on the indices $k$ and $h$ for ease of exposition. Assumption \ref{ass:truncated-subGaussian} imposes sub-Gaussian condition on the marginal probability density $\phi^1(0, \cdot)$ and all conditional probability densities.  Proposition \ref{lemma:gaussian-truncated-subG} in Appendix \ref{sec:proofs} verifies that every nondegenerate multivariate Gaussian distribution satisfies Assumption \ref{ass:truncated-subGaussian}. Notably, Assumption \ref{ass:truncated-subGaussian}  does not impose boundedness on the conditioning variable $\bz$, making it well suited for financial settings in which auxiliary time series or unbounded factors often serve as conditioning inputs \citep{mandelbrot1963variation, gopikrishnan2000statistical,cont2001empirical}.
{ This unboundedness of the conditioning variable places our setting beyond the framework of \cite{fu2024unveil}. We address this new challenge by leveraging an additional truncation step; see Theorem \ref{thm:score approx}.}
As a consequence of Assumption \ref{ass:truncated-subGaussian}, the following lemma shows that all joint distributions of interest are also~sub-Gaussian.

\begin{lemma}
    Under Assumption \ref{ass:truncated-subGaussian}, the probability density $p^{[1:h]}(\cdot)$ of $\bX_0^{[1:h]}$ is sub-Gaussian for each $h \in [H]$, i.e., there are positive constants $C_3, C_4 > 0$ such that $$p^{[1:h]}(\bx) \leq C_3\exp\parenthesis{-C_4\norm{\bx}^2/2}.$$
\end{lemma}

Similar to constants $C_1$ and $C_2$, the constants $C_3$ and $C_4$ are also $h$-dependent and we omit the dependence when it is clear from the context.

\begin{proof}
    We prove the lemma by induction on $h$. When $h = 1$, the claim is satisfied with $C_3 = C_1$ and $C_4 = C_2$. Assume the inductive hypothesis holds for $p^{[1:h-1]}(\cdot)$.
    Set $C_4' = \min\curly{C_4, C_2}$. For every $\bx \in \R^h$, it holds that
    \begin{align}
        p^{[1:h]}(\bx)
        &= p^{[1:h-1]}(\bx^{[1:h)})\phi^h_{\bx^{[1:h)}}(0, x^h) \nonumber \\
        &\leq C_3C_1 \exp\parenthesis{-\frac{C_4}{2}\norm{\bx^{[1:h)}}^2}\exp\parenthesis{-\frac{C_2}{2}\abs{x^h - \bv^\top \bx^{[1:h)}}^2} \nonumber \\
        &\leq C_3C_1 \exp\parenthesis{-\frac{C_4'}{2}\parenthesis{\norm{\bx^{[1:h)}}^2 + \abs{x^h - \bv^\top \bx^{[1:h)}}^2}}. \label{eq:joint-bound-1}
    \end{align}
    Define $\bM \in \R^{h \times h}$ as
    \begin{align*}
        \bM = \begin{pmatrix}
            \bv\bv^\top + \bI_{h-1} & -\bv\\ -\bv^\top & 1 
        \end{pmatrix}.
    \end{align*}
    We further bound \eqref{eq:joint-bound-1} with
    \begin{align*}
        p^{[1:h]}(\bx) \leq C_3C_1 \exp\parenthesis{-\frac{C_4'}{2} \bx^\top \bM \bx}.
    \end{align*}
    Here, the matrix $\bM$ is positive definite as long as the Schur complement of the top-left block $\bS \coloneqq  \bv\bv^\top + \bI_{h-1}$ is positive definite. To see this, the Woodbury identity implies that
    \begin{align*}
        \bM/\bS \coloneqq 1 - \bv^\top\parenthesis{\bv \bv^\top + \bI_{h-1}}^{-1}\bv = 1 - \bv^\top\parenthesis{\bI_{h-1} - \frac{\bv\bv^\top}{1 + \norm{\bv}^2}}\bv = \frac{1}{1 + \norm{\bv}^2}>0.
    \end{align*}
    Combining \eqref{eq:joint-bound-1} with $\bx^\top\bM\bx \geq \lambda_{\min}(\bM)\norm{\bx}^2$ yields
    \begin{align*}
        p^{[1:h]}(\bx) \leq C_3C_1\exp\parenthesis{-\frac{C_4'\lambda_{\min}(\bM)}{2}\norm{\bx}^2},
    \end{align*}
    which completes the proof. 
\end{proof}

By applying the marginalization technique, the sub-Gaussian property naturally holds for all joint distributions over any interval $[k,h]$. %

\begin{corollary}\label{coro:joint-subG}
    Under Assumption \ref{ass:truncated-subGaussian}, the probability density $p^{[k:h]}(\cdot)$ of $
    \bX_0^{[k:h]}$ is sub-Gaussian, i.e., there are constants $C_5, C_6 > 0$ such that $$p^{[k:h]}(\bx) \leq C_5\exp\parenthesis{-C_6\norm{\bx}^2/2}.$$
    In particular, probability density $p^h(\cdot)$ of $X_0^h$ is sub-Gaussian.
\end{corollary}

In addition to the light-tail property of each conditional distribution, we need the following assumption on the smoothness. Denote $\cH^\beta(\Omega, B)$ by the set of all $\beta$-Hölder smooth functions over the domain $\Omega$ with Hölder norm bounded by $B$; see \cite[Definition 2.1]{fu2024unveil}. %

\begin{assumption}[H\"{o}lder smoothness]\label{ass:holder-re}
    The marginal density at $h = 1$ and all conditional densities are Hölder smooth in the sense that $\phi^1(0, \cdot) \in\cH^\beta(\R, B)$ and $\phi_{\bz}^h(0, \cdot)\in\cH^\beta(\R^{h-k+1}, B)$ for all $1 \leq k < h \leq H$ and $\bz \in \R^{h-k}$. 
\end{assumption}

We are ready to establish score approximation theory. To this end, we rewrite the score function as $\partial_x \log \phi_\bz^h(t, x) = \frac{\partial_x \phi_\bz^h(t, x)}{\phi_\bz^h(t, x)} $. Following the ideas in \cite{fu2024unveil}, we first approximate  $\partial_x \phi_\bz^h(t, x)$ and $\phi_\bz^h(t, x)$ with diffused local polynomials. { The next result provides polynomial approximations for both the denominator $\phi_\bz^h(t,x)$ and the numerator $\sigma_t\partial_x\phi_\bz^h(t,x)$. Later, we combine these approximations to form a clipped ratio $f_3$ in \eqref{eq:f_3}, which is then approximated by a neural network.}

\begin{lemma}[Diffused local polynomial approximation] \label{lemma:poly-density}
Suppose Assumptions \ref{ass:truncated-subGaussian} and \ref{ass:holder-re} hold. For a sufficiently large integer $N>0$, there exist two polynomials $f_1(x, \bz, t)$ and $f_2(x, \bz, t)$, and a positive constant $R_1  \asymp \sqrt{\log N} $ such that
\begin{align}
    & \abs{f_1(x, \bz, t) - \phi^h_{\bz}(t, x) } \lesssim BN^{-\beta}\log^{\frac{\beta + 1}{2}}N, \label{eq:poly-1} \\ 
    & \abs{f_2(x, \bz, t) - \sigma_t\partial_x \phi_{\bz}^h(t, x)} \lesssim BN^{-\beta}\log^{\frac{\beta + 2}{2}}N, \label{eq:poly-2}
\end{align}
for any $x \in \R, \bz \in [-R_1, R_1]^{h-k}$ and $ t>0$.
\end{lemma}

{ The proof adapts the strategy of \cite[Appendix~A.4]{fu2024unveil}, with an additional truncation step to handle the unbounded conditioning variable. This truncation argument is established in Lemma~\ref{lemma:clip integral} in Appendix~\ref{sec:proofs}. The remaining calculations follow \cite[Appendix~A.4]{fu2024unveil}.}

\begin{proof}

 We first approximate the density $\phi^h_{\bz}(t, x)$ by an integral evaluated over a bounded domain. Denote the integral by
\begin{align}
    \tilde{f}_1(x, \bz, t) = \frac{1}{\sigma_t(2\pi)^{1/2}}\int_{\bar{B}}  \phi_{\bz}^h(0, y) \exp\parenthesis{-\frac{\abs{x - \alpha_t y}^2}{2\sigma_t^2}}\dd y,
\end{align}
where we define
\begin{align*}
    \bar{B} \coloneqq \bracket{\frac{x - \sigma_t R_2}{\alpha_t}, \frac{x + \sigma_t R_2}{\alpha_t}} \bigcap \bracket{\bv^\top \bz - R_3, \bv^\top \bz + R_3},
\end{align*}
with $\bv$ as defined in Assumption \ref{ass:truncated-subGaussian}.
Later, we will choose $R_2, R_3 \asymp \sqrt{\beta\log N}$ for some large enough positive integer $N$. Consequently, Lemma \ref{lemma:clip integral} in Section \ref{sec:proofs} implies that
\begin{align}
    \abs{\tilde{f}_1(x, \bz, t) - \phi^h_{\bz}(t, x)}  \leq N^{-\beta}, \;\; \forall x \in \R, \bz \in \R^{h-k}.
\end{align}
Compared to the steps in proving \cite[Lemma A.4]{fu2024unveil}, we only need to modify the function $f$ in (A.24) of \cite{fu2024unveil}. To proceed, we take
\begin{align}
    f(y, \bz, t) = \phi^h_{2R_1(\bz - {\bf 1}/2)}(0, 2R_1(y - 1/2) + \bv^\top \bz), \;\;  y \in [0, 1], \bz \in [0, 1]^{h - k}, t > 0,
\end{align}
where ${\bf 1} \in \R^{h-k}$ is the all one vector  and $R_1 \asymp \sqrt{\beta \log N}$. We calculate the Hölder norm of $f$ as follows. Let $g(\cdot)$ be an arbitrary $\beta$-Hölder smooth function. If $\norm{g}_{\cH^\beta(\R^d)} \leq B $, then for any $\bA \in \R^{d \times d}$ and $\bb \in \R^d$, it is straightforward to check by definition that 
\begin{align*}
    \norm{g(\bA \bx + \bb)}_{\cH^\beta(\R^d)} \leq \norm{\bA}^\beta \norm{g}_{\cH^\beta(\R^d)} \leq B\norm{\bA}^\beta, \;\; \forall \; \bx \in \R^d.
\end{align*}
Consider the linear transformation
\begin{align*}
    \begin{pmatrix}
        y \\ \bz 
    \end{pmatrix}
    \mapsto 
    \underbrace{\begin{pmatrix}
        2R_1 & \bv^\top \\ \bf{0} & 2R_1 \bI_{h-k}
    \end{pmatrix}}_\bA
    \begin{pmatrix}
        y \\ \bz 
    \end{pmatrix}
    \underbrace{- R_1 
    \begin{pmatrix}
        1 \\ \bf{1}
    \end{pmatrix}}_\bb.
\end{align*}
Since the operator norm of $A$ is bounded with
\begin{align*}
    \norm{\bA} \leq \norm{\bA}_F = \sqrt{(h-k)(2R_1)^2 + \norm{\bv}^2} \lesssim R_1,
\end{align*}
Assumption \ref{ass:holder-re} implies that $\norm{f}_{\cH^\beta} \lesssim BR_1^\beta$. By applying the construction of function $q(\cdot)$ in (A.26) of \cite{fu2024unveil} with substitution $\bx = y$ and $\by = \bz$, we have 
\begin{align}
    \abs{f(y, \bz, t) - q(y, \bz, t)} \lesssim B\frac{R_1^\beta d_{\rm trunc}^s}{s!N^\beta}.
\end{align}
Finally, we construct a diffused local polynomial as in (A.37) of \cite{fu2024unveil} and follow the same calculation in \cite[Appendix A.4]{fu2024unveil} to deduce
\begin{align*}
    \abs{f_1(x, \bz, t) - \phi^h_{\bz}(t, x) } \lesssim BN^{-\beta}\log^{\frac{\beta + 1}{2}}N, \;\; \forall x \in \R, \bz \in [-R_1, R_1]^{h-k}, t>0. 
\end{align*}
\end{proof}

Let $\epsilon_{\rm low}> 0$ be a threshold that will be chosen later and fix a large enough positive constant $C_7 >0  $. With the two polynomials defined as in Lemma \ref{lemma:poly-density}, we construct a score approximator as
\begin{align}
    f_3 = \min\parenthesis{\frac{f_2}{\sigma_t f_{1, \rm clip}}, \frac{C_7}{\sigma_t^2}\parenthesis{R_4 + 1}}, 
    \label{eq:f_3}
\end{align}
where $f_{1, \rm clip} \coloneqq \max(f_1, \epsilon_{\rm low})$ and $R_4 \asymp \sqrt{\log N}$. This construction guarantees that the denominator is bounded away from zero while the range of $f_3$ is also bounded. The next result establishes score approximation with $f_3$. The proof follows the same calculations as in \cite[Appendix A.3.2]{fu2024unveil}.

\begin{lemma}\label{lemma:f3}
Suppose Assumptions \ref{ass:truncated-subGaussian} and \ref{ass:holder-re} hold. Let $f_3$ be defined as in \eqref{eq:f_3}. For sufficiently large integer $N>0$, it holds that
\begin{align*}
    \abs{f_3(x, \bz, t) - \partial_x \log\phi_{\bz}^h(t, x)} \lesssim \frac{B}{\sigma_t^2 \phi_{\bz}^h(t, x)}N^{-\beta}\log^{\frac{\beta+2}{2}}N.
\end{align*}
for any $x \in \R, \bz \in [-R_1, R_1]^{h-k}$ and $ t>0$.
\end{lemma}

With Lemma \ref{lemma:f3}, we will complete the proof of score approximation as long as $f_3$ can be approximated by a neural network.
We make the following assumption on the neural network class $\cS^h$. 

\begin{assumption}\label{ass:net-class-1}
    Let $N > 0 $ be a given integer. There exists a neural network $f^{\rm net}(\cdot) \in \cS^h$ such that  $\norm{f^{\rm net}(\cdot, t)}_{\infty} \lesssim \frac{\sqrt{\log N}}{\sigma_t^2}$ for all $t \in [t_0, T]$ and moreover
    \begin{align*}
        \abs{f^{\rm net}(x, \bz, t) - f_3(x, \bz, t)} \lesssim N^{-\beta}, \; \forall x \in [-R_1, R_1], \bz \in [-R_1, R_1]^{h-k}, t\in [t_0, T].
    \end{align*}
\end{assumption}

Since $f_3$ is Lipschitz continuous over the compact domain, many universal approximation results \citep{bach2017breaking,kajitsuka2024are} guarantee the existence of neural networks as required in Assumption \ref{ass:net-class-1}. In Section \ref{sec:relu}, we will justify this assumption using ReLU networks as a concrete example.
We are ready to state the main result in this subsection.

\begin{theorem}[Score approximation]\label{thm:score approx}
       Suppose Assumptions \ref{ass:truncated-subGaussian},  \ref{ass:holder-re}, \ref{ass:net-class-1} hold. Let $N > 0$ be a sufficiently large integer and we set early-stop time $t_0 = N^{-C_\sigma}$ and terminal time $ T = C_\alpha \log N$ for some positive constants $C_\sigma$ and $C_\alpha$. There is a neural network $s^\star \in \cS^h$ such that
    \begin{align}\label{eq:univ-approx}
     \E\bracket{\abs{s^\star(t, X_t^h, \bX_0^{[k:h)}) - \partial_x \log \phi_{\bX_0^{[k:h)}}(t, X_t^h)}^2} \lesssim \frac{B}{\sigma_t^4}N^{-\beta}(\log N)^{2 + \beta/2}, \;\; \forall t \in [t_0, T]. 
\end{align}

\end{theorem}

Theorem \ref{thm:score approx} establishes a score approximation result over the time interval $[t_0, T]$. It will serve as a building block for the subsequent score estimation theory. Compared with prior work \citep{fu2024unveil}, the main technical difference is the introduction of an additional truncation step. This is necessary because the conditional information $\bX_0^{[k:h)}$ is not assumed to be bounded.

\begin{proof}
Let $s^\star(t, x, \bz) = f^{\rm net}(x, \bz, t)$ given in Assumption \ref{ass:net-class-1}. It follows from Lemma \ref{lemma:f3} that
\begin{align}
    \abs{s^\star(x, \bz, t) - \partial_x \log\phi_{\bz}(t, x)} \lesssim \frac{B}{\sigma_t^2 \phi_{\bz}(t, x)}N^{-\beta}\log^{\frac{\beta+2}{2}}N. \label{eq:approx trans}
\end{align}
when $x, \bz, t$ are bounded as in Assumption \ref{ass:net-class-1}.

Consider the following decomposition
\begin{align*}
    & \E\bracket{\abs{s^\star(t, X_t^h, \bX_0^{[k:h)}) - \partial_x \log \phi^h_{\bX_0^{[k:h)}}(t, X_t^h)}^2} \\ & \leq \underbrace{\E\bracket{1\curly{\norm{\bX_0^{[k:h)}}_\infty > R_1}\abs{s^\star(t, X_t^h, \bX_0^{[k:h)}) - \partial_x \log \phi^h_{\bX_0^{[k:h)}}(t, X_t^h)}^2}}_{(E_1)} + \\ & \qquad + \underbrace{\E\bracket{1\curly{\norm{\bX_0^{[k:h)}}_\infty \leq R_1, \abs{X_t^{h} - \alpha_t \bv^\top \bX_0^{[k:h)}} > R_5}\abs{s^\star(t, X_t^h, \bX_0^{[k:h)}) - \partial_x \log \phi^h_{\bX_0^{[k:h)}}(t, X_t^h)}^2}}_{(E_2)} \\ & \qquad + \underbrace{\E\bracket{1\curly{\norm{\bX_0^{[k:h)}}_\infty \leq R_1, \abs{X_t^{h}- \alpha_t \bv^\top \bX_0^{[k:h)}} \leq R_5}\abs{s^\star(t, X_t^h, \bX_0^{[k:h)}) - \partial_x \log \phi^h_{\bX_0^{[k:h)}}(t, X_t^h)}^2}}_{(E_3)},
\end{align*}
where $R_1 $ and $ R_5$ will be chosen later. 

 We begin with the error term $(E_1)$. Note that
 \begin{align*}
    (E_1) \leq 2\E\bracket{1\curly{\norm{\bX_0^{[k:h)}}_\infty > R_1}\abs{s^\star(t, X_t^h, \bX_0^{[k:h)})}^2} + 2\E\bracket{1\curly{\norm{\bX_0^{[k:h)}}_\infty > R_1}\abs{\partial_x \log \phi^h_{\bX_0^{[k:h)}}(t, X_t^h)}^2}.
\end{align*}
 As $\abs{s^\star(x, \bz, t)} \lesssim \frac{\sqrt{\log N}}{\sigma_t^2}$ by Assumption \ref{ass:net-class-2}, we have
\begin{align}
\E\bracket{1\curly{\norm{\bX_0^{[k:h)}}_\infty > R_1}\abs{s^\star(t, X_t^h, \bX_0^{[k:h)})}^2} & \lesssim \frac{\log N}{\sigma_t^4} \int_{\norm{\bz}_\infty > R_1}p^{[k:h)}( \bz) \dd \bz \nonumber
\\ & \lesssim \frac{\log N}{\sigma_t^4}\int_{\norm{\bz}_\infty > R_1}\exp\parenthesis{-C_5\norm{\bz}^2/2}\dd \bz\nonumber \\ & \lesssim \frac{\log N}{\sigma_t^4} (R_1)^{-1}\exp\parenthesis{-C_5R_1^2/2},\label{eq:E1-1}
\end{align}
where $C_5$ is given in Corollary \ref{coro:joint-subG}.

Combining \eqref{eq:E1-1} and Lemma \ref{lemma:score-tail}, and choosing $R_1 \geq \sqrt{\frac{4\beta}{\min(C_5, C_6)} \log N}$, we bound the term $(E_1)$ with
\begin{align}
    (E_1) \lesssim \frac{\log N}{\sigma_t^4} (R_1)^{-1}\exp\parenthesis{-C_5R_1^2/2} + \frac{1}{\sigma_t^4} \exp(-C_6 R_1^2/2) \lesssim \frac{\sqrt{\log N}}{\sigma_t^4}N^{-2\beta}. \label{eq:E1-bound}
\end{align}

Next, we bound $(E_2)$ with
\begin{align}
    (E_2) & \leq 2\E\bracket{1\curly{\norm{\bX_0^{[k:h)}}_\infty \leq R_1, \abs{X_t^{h} - \alpha_t \bv^\top \bX_0^{[k:h)}} > R_5}\abs{s^\star(t, X_t^h, \bX_0^{[k:h)}) }^2}\nonumber \\ & \qquad + 2\E\bracket{1\curly{\norm{\bX_0^{[k:h)}}_\infty \leq R_1, \abs{X_t^{h}- \alpha_t \bv^\top \bX_0^{[k:h)}} > R_5}\abs{\partial_x \log \phi_{\bX_0^{[k:h)}}(t, X_t^h)}^2} \label{eq:E2-1}
\end{align}
Lemma \ref{lemma:E2-terms} implies
\begin{align} & \E\bracket{1\curly{\norm{\bX_0^{[k:h)}}_\infty \leq R_1, \abs{X_t^{h} - \alpha_t \bv^\top \bX_0^{[k:h)}} > R_5}\abs{s^\star(t, X_t^h, \bX_0^{[k:h)}) }^2}  \nonumber \\ & \lesssim \frac{\log N}{\sigma_t^4}\int_{\norm{\bz}_\infty \leq R_1}\int_{\abs{x - \alpha_t \bv^\top \bz}> R_5} \phi^h_{\bz}(t, x) \dd x p^{[k:h)}( \bz)  \dd \bz \nonumber \\ & \lesssim \frac{\log N}{\sigma_t^4}(R_5)^{-1}\exp\parenthesis{-C_2' R_5^2/2},  \label{eq:E2-2}
\end{align}
and moreover
\begin{align} & \E\bracket{1\curly{\norm{\bX_0^{[k:h)}}_\infty \leq R_1, \abs{X_t^{h}- \alpha_t \bv^\top \bX_0^{[k:h)}} > R_5}\abs{\partial_x \log \phi^h_{\bX_0^{[k:h)}}(t, X_t^h)}^2} \nonumber \\ & =  \int_{\norm{\bz}_\infty \leq R_1}\int_{\abs{x - \alpha_t \bv^\top \bz}> R_5} \phi^h_{\bz}(t, x)\abs{\partial_x \log \phi^h_{\bz}(t, x)}^2 \dd x p^{[k:h)}(\bz)  \dd \bz  \nonumber \\ & \lesssim \frac{1}{\sigma_t^4}R_5\exp\parenthesis{-C_2'R_5^2/2}. \label{eq:E2-3}
\end{align}
With \eqref{eq:E2-1}, \eqref{eq:E2-2} and \eqref{eq:E2-3}, we choose $R_5 \geq \sqrt{\frac{4\beta}{C_2'}\log N}$ to obtain an upper bound for $(E_2)$:
\begin{align}
    (E_2) & \lesssim {\sigma_t^{-4}}(R_5)^{-1}\exp\parenthesis{-C_2'R_5^2/2} + \frac{1}{\sigma_t^4}R_5\exp\parenthesis{-C_2'R_5^2/2} \nonumber \\ & \lesssim \frac{\log N}{\sigma_t^4}\cdot \frac{N^{-2\beta}}{\sqrt{\log N}} + \frac{\sqrt{\log N}}{\sigma_t^4}N^{-2\beta} \lesssim \frac{\sqrt{\log N}}{\sigma_t^4}N^{-2\beta}. \label{eq:E2-bound}
\end{align}

We further bound $(E_3) \leq (E_4) + (E_5)$ with
\begin{align*}
    & (E_4)  \coloneqq \E\bigg [1\curly{\norm{\bX_0^{[k:h)}}_\infty \leq R_1, \abs{X_t^{h} - \alpha_t \bv^\top \bX_0^{[k:h)}} \leq R_5, \abs{\phi_{\bX_0^{[k:h)}}(t, X_t^h)} < \epsilon_{\rm low}} \\ & \qquad\qquad\qquad\qquad\times \abs{s^\star(t, X_t^h, \bX_0^{[k:h)}) - \partial_x \log \phi_{\bX_0^{[k:h)}}(t, X_t^h)}^2\bigg] \\ 
   & (E_5)  \coloneqq \E\bigg [1\curly{\norm{\bX_0^{[k:h)}}_\infty \leq R_1, \abs{X_t^{h} - \alpha_t \bv^\top \bX_0^{[k:h)}} \leq R_5, \abs{\phi_{\bX_0^{[k:h)}}(t, X_t^h)} \geq \epsilon_{\rm low}} \\ & \qquad\qquad\qquad\qquad\times \abs{s^\star(t, X_t^h, \bX_0^{[k:h)}) - \partial_x \log \phi_{\bX_0^{[k:h)}}(t, X_t^h)}^2\bigg]
\end{align*}
We bound $(E_4)$ with
\begin{align*}
    (E_4)
    & \leq  2\E\bracket{
        1\curly{\substack{
        \norm{\bX_0^{[k:h)}}_\infty \leq R_1,\ 
        \abs{X_t^{h} - \alpha_t \bv^\top \bX_0^{[k:h)}} \leq R_5,\\
        \abs{\phi_{\bX_0^{[k:h)}}(t, X_t^h)} < \epsilon_{\rm low}}}
        \abs{s^\star(t, X_t^h, \bX_0^{[k:h)})}^2} \\
    & \qquad + 2\E\bracket{
        1\curly{\substack{
        \norm{\bX_0^{[k:h)}}_\infty \leq R_1,\ 
        \abs{X_t^{h} - \alpha_t \bv^\top \bX_0^{[k:h)}} \leq R_5,\\
        \abs{\phi_{\bX_0^{[k:h)}}(t, X_t^h)} < \epsilon_{\rm low}}}
        \abs{\partial_x \log \phi_{\bX_0^{[k:h)}}(t, X_t^h)}^2},
\end{align*}
where  $\epsilon_{\rm low} > 0$ will be chosen later.
Lemma \ref{lemma:small density} implies the following upper bound for $(E_4)$:
\begin{align}
    (E_4)  \lesssim  \int_{\norm{\bz}_\infty \leq R_1}\parenthesis{\frac{\log N}{\sigma_t^4} R_5 \epsilon_{\rm low}  + \frac{\epsilon_{\rm low}}{\sigma_t^4} R_5^3 } p^{[k:h)}(\bz)  \dd \bz     \lesssim \frac{\epsilon_{\rm low}}{\sigma_t^4} \cdot (\log N)^{3/2}.\label{eq:E4-bound}
\end{align}
By the approximation guarantee \eqref{eq:approx trans}, for any fixed $\bz \in [-R_1, R_1]^{h-k}$ we deduce that
\begin{align*}
& \int_{\abs{x - \alpha_t \bv^\top \bz}\leq  R_5}1\curly{\abs{\phi^h_{\bz}(t, x)} \geq \epsilon_{\rm low}} \abs{s^\star(x, \bz, t) - \partial_x \log\phi^h_{\bz}(t, x)}^2\phi^h_{\bz}(t, x) \dd x   \\ & \lesssim \int_{\abs{x - \alpha_t \bv^\top \bz}\leq  R_5}1\curly{\abs{\phi^h_{\bz}(t, x)} \geq \epsilon_{\rm low}}\frac{B^2}{\sigma_t^4 \phi^h_{\bz}(t, x)}N^{-2\beta}\log^{\beta+2}N \dd x \\ & \lesssim \frac{B^2 \log^{\beta + 2}N}{\sigma_t^4\epsilon_{\rm low}}N^{-2\beta}R_5.
\end{align*}
Consequently, we obtain an upper bound for $(E_5)$:
\begin{align}
    (E_5) \lesssim  \frac{B^2 \log^{\beta + 2}N}{\sigma_t^4\epsilon_{\rm low}}N^{-2\beta}R_5 \int_{\norm{\bz}_\infty \leq R_1}p^{[k:h)}(\bz)  \dd \bz \leq \frac{B^2 \log^{5/2 + \beta}N}{\sigma_t^4\epsilon_{\rm low}}N^{-2\beta}. \label{eq:E5-bound}
\end{align}
Combining \eqref{eq:E1-bound}, \eqref{eq:E2-bound}, \eqref{eq:E4-bound} and \eqref{eq:E5-bound}, we have 
\begin{align}
    \E\bracket{\abs{s^\star(t, X_t^h, \bX_0^{[k:h)}) - \partial_x \log \phi^h_{\bX_0^{[k:h)}}(t, X_t^h)}^2} \lesssim \frac{\sqrt{\log N}}{\sigma_t^4}N^{-2\beta} + \frac{\epsilon_{\rm low}(\log N )^{3/2}}{\sigma_t^4} + \frac{B^2 \log^{5/2 + \beta}N}{\sigma_t^4\epsilon_{\rm low}}N^{-2\beta}. \label{eq:L2-approx-1}
\end{align}
Substitution  $\epsilon_{\rm low} = B N^{-\beta}\log^{\frac{1+\beta}{2}}N$ back into \eqref{eq:L2-approx-1}, we conclude 
\begin{align}
    \E\bracket{\abs{s^\star(t, X_t^h, \bX_0^{[k:h)}) - \partial_x \log \phi^h_{\bX_0^{[k:h)}}(t, X_t^h)}^2} \lesssim \frac{B}{\sigma_t^4}N^{-\beta}(\log N)^{2 + \beta/2}, \;\; \forall t \in [t_0, T]. 
\end{align}
\end{proof}

\subsection{Score Estimation}\label{sec:score-est}
In this section, we present the score estimation result for one-step conditional diffusion models. Define the risk function as
    \begin{align}
    \cR^{[k:h]}(s) = \int_{t_0}^T \frac{1}{T - t_0} \E_{\bX_0^{[k:h)}}\E_{X_0^{h} | \bX_0^{[k: h)}}\E_{X_t^h\vert X_0^h}\bracket{\abs{s(t, X_t^h, \bX_0^{[k:h)}) -  \partial_x \log \phi_{\bX_0^{[k:h)}}(t, X_t^h)}^2}\dd t. \label{eq:risk}
    \end{align}
Our goal is to bound expected risk $\E_\cD[\cR^{[k:h]}(\hat{s})]$, where we recall that $\hat{s}^h$ is a minimizer to the empirical score matching risk $\widehat{\cL}^{[k:h]}(s) = \frac{1}{n}\sum_{i = 1}^n \ell^{[k:h]}(\bx_i^{[k:h]}; s)$ with $\ell^{[k:h]}(\cdot; s) $  defined in \eqref{eq:loss-lh}. In the subsequent analysis, we suppress the superscript $[k:h]$ when the context is clear. Recall the dataset $\cD = \curly{\bx_i}_{i = 1}^n$. 

One key step towards the score estimation theory is bounding the logarithmic covering number of a function class over a compact domain. To this end, it is convenient to recenter the empirical by~defining 
\begin{align}
    \ell_0(x, \bz; s) = \ell(x, \bz; s) - \ell(x, \bz; \bar{s}), \;\; \text{with} \;\; \bar{s}(t, x, \bz) = \partial_x \log \phi_\bz^h(t, x).  \label{eq:centered loss}
\end{align}
Moreover, we define a truncated loss associated with $\ell(\cdot; s)$ as
\begin{align}
    \ell^{\rm trunc}(x, \bz; s) = \ell(x, \bz; s)1\curly{\abs{x - \bv^\top \bz} \leq R_3, \norm{\bz}_\infty \leq R_1}. \label{eq:truncated loss}
\end{align}
The truncated centered loss $\ell_0^{\rm trunc}(\cdot; s)$ is defined in a similar way. 
Let $\cG = \curly{\ell_0^{\rm trunc}(\cdot; s): s \in \cS^h}$ be a loss function class induced by the neural network class $\cS^h$. We make the following definition on the network complexity. 

\begin{definition}[Network Complexity]\label{ass:net-class-2}
    We denote \texttt{Network\allowbreak-Complexity} by the logarithmic covering number $\log(\mathscr{N}(\delta))$.
\end{definition}

We are ready to state the main result of this subsection.

\begin{theorem}[Score estimation]\label{thm:score est}
     Suppose Assumptions \ref{ass:truncated-subGaussian} and \ref{ass:holder-re} hold and let $\cS^h$ be a neural network class satisfying Assumption \ref{ass:net-class-1}. Then it holds that
    \begin{align*}
        \E_{\cD}[\cR(\hat{s})] \lesssim \frac{\sqrt{\log N}}{t_0}N^{-\beta} + \frac{1}{t_0}N^{-\beta}(\log N)^{2+ \beta/2}  + \frac{\log N}{nt_0}\cdot \texttt{Network-Complexity} + \delta.
    \end{align*}

\end{theorem}

{ Theorem~\ref{thm:score est} bounds the score estimation error by four error sources: the first term is the truncation error from restricting the unbounded conditioning variable to a bounded domain; the second term is the score approximation error inherited from Theorem~\ref{thm:score approx}; the third term is the statistical error over a prescribed network class; and  the fourth term $\delta$ is the covering resolution of $\cG$. We emphasize that the analysis is {\it agnostic} to the choice of network class: any class satisfying Assumption~\ref{ass:net-class-1} can be substituted into Theorem~\ref{thm:score est}. Section \ref{sec:relu} instantiates the ReLU networks to obtain a concrete rate.}

\begin{proof}

    We follow the proof strategy as in \cite[Appendix C.1]{zhang2026diffusion} and only mention the differences here. With $\ell_0(\cdot; s)$ in \eqref{eq:centered loss} and the dataset $\cD$, we define centered empirical risk as
\begin{align*}
    \widehat{\cL}_0(s) \coloneqq \frac{1}{n}\sum_{i=1}^n \ell_0(x_i^h, \bx_i^{[k:h)}; s)
\end{align*}
Note that $\widehat{\cL}_0(s)$ and $\widehat{\cL}(s)$ share the same minimizer $\hat{s}$. One key observation is that population centered risk is exactly the risk we would like to minimize, i.e., 
\begin{align}
    \cL_0(s) \coloneqq \E_{\bX_0^{[k:h]}}\bracket{\ell_0(X_{0}^h, \bX_{0}^{[k:h)}, s)} = \cR(s).  \label{eq:L0-R}
\end{align}
We define truncated versions $\cL_0^{\rm trunc}$ and $\widehat{\cL}_0^{\rm trunc}$ using $\ell_0^{\rm trunc}(\cdot; s)$ accordingly. 

Consider the following decomposition
\begin{align*}
    \E_\cD[\cR(\hat{s})] = \underbrace{\E_\cD[\cR(s^\star)]}_{\rm (I)} + \underbrace{\E_\cD[\cR(\hat{s}) - \cR(s^\star)]}_{\rm (II)}.
\end{align*}
Here, $s^\star$ is the constructed score approximator given in Theorem \ref{thm:score approx} with approximation error:
\begin{align}
    {\rm (I)} \lesssim \frac{1}{T - t_0}\int_{t_0}^T \frac{B}{\sigma_t^4}N^{-\beta}(\log N)^{2 + \beta/2} \dd t \lesssim \frac{1}{t_0}N^{-\beta}(\log N)^{2+ \beta/2}, \label{eq:score-est-1}
\end{align}
where we apply Lemma D.1 in \cite{fu2024unveil} to bound the integral $\frac{1}{T-t_0}\int_{t_0}^T \sigma_t^{-4}\dd t = \cO(1/t_0)$. We proceed with the calculation as in \cite{zhang2026diffusion}. By applying the fact that $\hat{s}$ is the minimizer of $\widehat{\cL}_0$, we further bound ${\rm (II)}$ with
\begin{align*}
    {\rm (II)} & \leq \underbrace{\E_\cD[\cL_0(\hat{s}) - \widehat{\cL}_0(\hat{s})]}_{\rm (II-A)}.
\end{align*}

\paragraph{Bounding term (II-A).}
This part is an extension from a bounded domain to an unbounded domain with truncation argument. The key difference compared to \cite{zhang2026diffusion} is the truncation error and the covering number. Specifically, we first have
\begin{align}
    \text{(II-A)} \leq \underbrace{\E_{\cD, \bar{\cD}}\bracket{\sup_{s\in \cS}\frac{1}{n}\sum_{i=1}^n[\ell_0(\bar{x}_{i}^h, \bar{\bx}_{i}^{[k:h)}, s) - \ell_0(x_{i}^h, \bx_{i}^{[k:h)}, s)] - a\cR(s)}}_{(\spadesuit)} + a\E_\cD[\cR(\hat{s})], \label{eq:score-est-3}
\end{align}
where $a > 0$ is a parameter to control bias-variance trade-off commonly used in the literature. 
Here, $\bar{\cD} \coloneqq \curly{\bar{\bx}_{i}^{[k:h]}}_{i=1}^n$ is an independent copy of samples following the same distribution as $\cD$. We use the covering number to bound $(\spadesuit)$.  We construct a covering of $\cG$ in $L^\infty$-norm as follows. Let $\delta > 0$ be given. We select a collection of score networks $s_j, j = 1, \dots, \cN(\delta)$ such that for any $s \in \cS$ there is an index $j$ with $$\norm{\ell_0^{\rm trunc}(\cdot; s) - \ell_0^{\rm trunc}(\cdot; s_j)}_\infty = \norm{\ell^{\rm trunc}(\cdot; s) - \ell^{\rm trunc}(\cdot; s_j)}_\infty \leq \delta.$$ 
Recall that $B_{2}  = \bracket{\bv^\top \bz - R_3, \bv^\top \bz + R_3}$.
Note that for each $s \in \cS$ ($s$ can depend on $\bX_0^{[k:h]}$) we have
\begin{align*}
    & (T - t_0)\E_{\bX_0^{[k:h]}}\bracket{\abs{\ell(X_0^h, \bX_0^{[k:h)}; s) - \ell^{\rm trunc}(X_0^h, \bX_0^{[k:h)}; s)}} \\ & = \int_{t_0}^T \int_{\norm{\bz}_\infty > R_1}\int_{\R}\E_{X_t^h\vert X_0^h=y}\bracket{\abs{s(t, X_t^h, \bz) -  \frac{\alpha_t y - X_t^h}{\sigma_t^2}}^2}\phi_{\bz}^h(0, y)p^{[k:h)}(\bz)\dd y \dd \bz\dd t \\ & \qquad + \int_{t_0}^T \int_{\norm{\bz}_\infty \leq R_1}\int_{\R \setminus B_2}\E_{X_t^h\vert X_0^h=y}\bracket{\abs{s(t, X_t^h, \bz) -  \frac{\alpha_t y - X_t^h}{\sigma_t^2}}^2}\phi^h_{\bz}(0, y)p^{[k:h)}(\bz)\dd y \dd \bz\dd t\\ & \leq 2 \int_{t_0}^T \int_{\norm{\bz}_\infty > R_1}\int_{\R}\E_{X_t^h\vert X_0^h=y}\bracket{\abs{s(t, X_t^h, \bz)}^2 +  \abs{\frac{\alpha_t y - X_t^h}{\sigma_t^2}}^2}\phi^h_{\bz}(0, y)p^{[k:h)}(\bz)\dd y \dd \bz\dd t \\ & \qquad +  2 \int_{t_0}^T \int_{\norm{\bz}_\infty \leq R_1}\int_{\R \setminus B_2}\E_{X_t^h\vert X_0^h=y}\bracket{\abs{s(t, X_t^h, \bz)}^2 +  \abs{\frac{\alpha_t y - X_t^h}{\sigma_t^2}}^2}\phi^h_{\bz}(0, y)p^{[k:h)}(\bz)\dd y \dd \bz\dd t.
\end{align*}
Since $\phi_\bz^h(0, \cdot)$ and $p^{[k:h)}(\cdot)$ are both sub-Gaussian (Assumption \ref{ass:truncated-subGaussian} and Corolary \ref{coro:joint-subG}), we deduce that
\begin{align*}
    & (T - t_0)\E_{\bX_0^{[k:h]}}\bracket{\abs{\ell(X_0^h, \bX_0^{[k:h)}; s) - \ell^{\rm trunc}(X_0^h, \bX_0^{[k:h)}; s)}} \\ & \leq 2\int_{t_0}^T \int_{\norm{\bz}_\infty > R_1} \parenthesis{\frac{\log N}{\sigma_t^4} + \frac{1}{\sigma_t^2}} \exp(-C_6 \norm{\bz}^2/2)  \dd \bz\dd t \\ & \qquad +  2\int_{t_0}^T \int \int_{\abs{y - \bv^\top \bz} > R_3} \parenthesis{\frac{\log N}{\sigma_t^4} + \frac{1}{\sigma_t^2}} \exp(-C_2 \abs{y - \bv^\top \bz}^2/2) \dd y \dd \bz\dd t\\ & \leq 2(1+\log N) \big( \exp(-C_6 R_1^2/2)(R_1)^{-1} + \exp(-C_2 R_3^2/2)(R_3)^{-1} \big)\int_{t_0}^T\frac{1}{\sigma_t^4}\dd t 
\end{align*}
We take $R_1 \geq \sqrt{\frac{2\beta}{C_6}\log N}$ and $R_3 \geq \sqrt{\frac{2\beta}{C_2}\log N}$ to  deduce
\begin{align}
    \E_{\bX_0^{[k:h]}}\bracket{\abs{\ell(X_0^h, \bX_0^{[k:h)}; s) - \ell^{\rm trunc}(X_0^h, \bX_0^{[k:h)}; s)}} \lesssim \frac{\sqrt{\log N}}{t_0}N^{-\beta}. \label{eq:l-trunc-diff}
\end{align}

For any $s \in \cS^h$ and its close representation $s_j$ in terms of $\ell_0^{\rm trunc}$, it follow from \eqref{eq:l-trunc-diff} that
\begin{align*}
    & \E_{\bX_0^{[k:h]}}\bracket{\abs{\ell_0(X_0^h, \bX_0^{[k:h)}, s) - \ell_0(X_0^h, \bX_0^{[k:h)}, s_j)}}  = \E_{\bX_0^{[k:h]}}\bracket{\abs{\ell(X_0^h, \bX_0^{[k:h)}, s) - \ell(X_0^h, \bX_0^{[k:h)}, s_j)}} \\ & \leq \E_{\bX_0^{[k:h]}}\bracket{\abs{\ell(X_0^h, \bX_0^{[k:h)}, s) - \ell^{\rm trunc}(X_0^h, \bX_0^{[k:h)}, s)}}  + \E_{\bX_0^{[k:h]}}\bracket{\abs{\ell^{\rm trunc}(X_0^h, \bX_0^{[k:h)}, s) - \ell^{\rm trunc}(X_0^h, \bX_0^{[k:h)}, s_j)}} \\ & \qquad + \E_{\bX_0^{[k:h]}}\bracket{\abs{\ell^{\rm trunc}(X_0^h, \bX_0^{[k:h)}, s_j) - \ell(X_0^h, \bX_0^{[k:h)}, s_j)}} \\ & \lesssim \frac{2\sqrt{\log N}}{t_0}N^{-\beta} + \delta. 
\end{align*}
Furthermore, Eq.~\eqref{eq:L0-R} implies
\begin{align*}
    \abs{\cR(s) - \cR(s_j)} = \abs{\cL_0(s) - \cL_0(s_j)} & = \E_{\bX_0^{[k:h]}}\bracket{\abs{\ell_0(X_0^h, \bX_0^{[k:h)}, s) - \ell_0(X_0^h, \bX_0^{[k:h)}, s_j)}}  \\ &  \leq \frac{2\sqrt{\log N}}{t_0}N^{-\beta} + \delta.
\end{align*}
Consequently, Eq.~\eqref{eq:score-est-3} becomes
\begin{align*}
    (\spadesuit) \leq \E_{\cD, \bar{\cD}}\bracket{\max_{j}\frac{1}{n}\sum_{i=1}^n[\ell_0(\bar{X}_{0, i}^h, \bar{X}_{0, i}^{[k:h)}, s_j)- \ell_0(X_{0, i}^h, X_{0, i}^{[k:h)}, s_j)] - a\cR(s_j)} + (a+2)\parenthesis{\frac{2\sqrt{\log N}}{t_0}N^{-\beta} + \delta}.
\end{align*}

To proceed, we prove a modified version of Lemma C.1 in \cite{zhang2026diffusion} over an unbounded domain. Note that for any fixed $s \in \cS^h$, it holds that
\begin{align*}
      \ell(y, \bz; s) & \leq \frac{2}{T- t_0}\int_{t_0}^T\E_{X_t^h\vert X_0^h=y}\bracket{\abs{s(t, X_t^h, \bz)}^2 +  \abs{\frac{\alpha_t y - X_t^h}{\sigma_t^2}}^2}\dd t \\ & \leq \frac{2}{T- t_0}\int_{t_0}^T\E_{X_t^h\vert X_0^h=y}\bracket{\frac{\log N}{\sigma_t^4} +  \frac{1}{\sigma_t^2}}\dd t \\ & \leq \frac{2(1+ \log N)}{T-t_0}\int_{t_0}^T \frac{1}{\sigma_t^4}\dd t \leq \frac{4(1+ \log N)}{t_0},
\end{align*}
Similarly, we have
\begin{align}
    \abs{\ell(y, \bz; s) - \ell(y, \bz; \bar{s}) }^2 \leq \frac{16(1+\log N)}{t_0(T-t_0)}\int_{t_0}^T\E_{X_t^h\vert X_0^h=y}\bracket{\abs{(s(t, x, \bz) - \bar{s}(t, x, \bz))}^2}.\label{eq:l-diff}
\end{align}
Taking expectation over \eqref{eq:l-diff} to obtain
\begin{align*}
    & \E_{\bX_0^{[k:h]}}\bracket{\abs{\ell(X_0^h, \bX_0^{[k:h)}; s) - \ell(X_0^h, \bX_0^{[k:h)}; \bar{s}) }^2} \\ & \leq \frac{16(1+\log N)}{t_0(T-t_0)}\int_{t_0}^T \E\bracket{\abs{(s(t, x, \bz) - \bar{s}(t, x, \bz))}^2}\dd t = \frac{16(1 + \log N)}{t_0}\cR(s).
\end{align*}

Denote $h_i(s) = \ell_0(\bar{x}_{i}^h, \bar{\bx}_{i}^{[k:h)}, s_j) - \ell_0(x_{ i}^h, \bx_{i}^{[k:h)}, s_j)$ for each fixed $s \in \cS^h$. Let $C_\ell \coloneqq 16(1+\log N)/t_0$. We have
\begin{align*}
    {\rm Var}[h_i(s)] \leq C_\ell \cR(s), \;\; \text{and} \;\; \abs{h_i(s)} \leq C_\ell. 
\end{align*}
The analysis in \cite{zhang2026diffusion} implies
\begin{align*}
    (\spadesuit) \leq \frac{1}{\lambda}\log \mathscr{N}(\delta) + (a+2)\parenthesis{\frac{2\sqrt{\log N}}{t_0}N^{-\beta} + \delta}.
\end{align*}
with $\lambda = \frac{6an}{(2a+3)C_\ell}$. Hence, we apply Definition \ref{ass:net-class-2} to deduce
\begin{align}
    (\spadesuit) \lesssim \frac{(2a+3)C_\ell}{6an}\cdot \log \mathscr{N}(\delta) + (a+2)\parenthesis{\frac{2\sqrt{\log N}}{t_0}N^{-\beta} + \delta}. \label{eq:score-est-4}
\end{align}
Combining \eqref{eq:score-est-1}, \eqref{eq:score-est-3} and \eqref{eq:score-est-4}, we have
\begin{align*}
    \E_\cD[\cR(\hat{s})] - a\E_\cD[\cR(\hat{s})] & \lesssim \frac{1}{t_0}N^{-\beta}(\log N)^{2+ \beta/2} + \frac{(2a+3)C_\ell}{6an}\cdot \log \mathscr{N}(\delta) \\ & \qquad + (a+2)\parenthesis{\frac{2\sqrt{\log N}}{t_0}N^{-\beta} + \delta}. 
\end{align*}
We pick $a = 1/2$ and deduce that
\begin{align*}
    \E_\cD[\cR(\hat{s})] & \lesssim \frac{\sqrt{\log N}}{t_0}N^{-\beta} + \frac{1}{t_0}N^{-\beta}(\log N)^{2+ \beta/2}  + \frac{\log N}{nt_0}\cdot \log \mathscr{N}(\delta) + \delta. 
\end{align*}
Therefore, we complete the proof. 
\end{proof}

\subsection{Distribution Estimation}\label{sec:dist-est}
In this subsection, we apply the score estimation theory from the previous subsection to obtain one-step and joint distribution estimation theory. Recall $d_{\rm trunc} = h - k + 1$.

\begin{theorem}[Distribution estimation]\label{thm:dist est}
Let $h \in [H]$ and $k \in [h-1]$.  Assume that
\[
\E_{\bX_0^{[k:h)}}[{\rm KL}(\mu_0^h[\bX_0^{[k:h)}]\| \cN(0, 1))] < \infty .
\]
Under Assumptions \ref{assumption:dependence_decay}, \ref{ass:truncated-subGaussian}, \ref{ass:holder-re}, \ref{ass:net-class-1}, it holds that
\begin{align*}
    \E_{\cD} [\E_{\bX_0^{[1:h)}}[{\rm TV}(\mu_0^{h}[\bX_0^{[1:h)}], \hat{\mu}_{t_0}^{h}[\bX_0^{[1:h)}])]]  \lesssim \sqrt{t_0}(\log(1/t_0))^{\frac{d_{\rm trunc}}{2}} + \exp(-T) + \sqrt{T\E_{\cD}[\cR(\hat{s})]} + u(d_{\rm trunc}).
\end{align*}
\end{theorem}

{ Theorem~\ref{thm:dist est} translates the score estimation result in Theorem~\ref{thm:score est} into a distribution estimation guarantee. The first term is the early-stopping error from terminating the sampling process at $T-t_0\in(0,T)$, and the second is the initialization error from starting the sampling process at the standard Gaussian rather than the true forward marginal at time $T$. The third term propagates the score estimation error through the sampling process, with $\E_{\cD}[\cR(\hat s)]$ controlled by Theorem~\ref{thm:score est}. The last term, $u(d_{\rm trunc})$, is the truncation error, controlled by Assumption~\ref{assumption:dependence_decay}.}

\begin{proof}
    The proof largely follows from \cite[Appendix D.3]{fu2024unveil}. We first check that condition (D.20) in \cite[Lemma D.4]{fu2024unveil} holds. Note that for any $s \in \cS^h$ and $\bz \in \R^{h - k}$, Assumption \ref{ass:net-class-1}, Lemmas \ref{lemma:tail-density-t} and \ref{lemma:score-bound} imply that
    \begin{align*}
        & \int \phi^h_{\bz}(t, x) \abs{s(t, x, \bz) - \partial_x \log\phi^h_{\bz}(t, x)}^2 \dd x  \\ & \lesssim  \int \phi^h_{\bz}(t, x)\frac{\abs{x - \alpha_t \bv^\top \bz}^2 + \log N}{\sigma_t^4} \dd x \\ & \lesssim \int \exp\parenthesis{-\frac{C_2'}{2}(x - \alpha_t \bv^\top \bz)^2} \frac{\abs{x - \alpha_t \bv^\top \bz}^2 + \log N}{\sigma_t^4} \dd x \\ & = \int \exp\parenthesis{-\frac{C_2'}{2}u^2} \frac{u^2 + \log N}{\sigma_t^4} \dd u \lesssim \frac{1}{\sigma_t^4}.
    \end{align*}
    
    We also prove an analog to \cite[Lemma D.5]{fu2024unveil} to bound the early stopping error. Let $\bz \in \R^{h-k}$ be given. Recall in Lemma \ref{lemma:clip integral} we define $B_2 = \bracket{\bv^\top \bz - R_3, \bv^\top \bz + R_3}$ and furthermore we denote $B_3 \coloneqq \bracket{\alpha_t \bv^\top \bz - R_6, \alpha_t \bv^\top \bz + R_6}$. Then we first decompose the TV distance as
    \begin{align}
        {\rm TV}(\mu_0^{h}[\bz], \mu_{t_0}^{h}[\bz]) & = \frac{1}{2}\int_{(B_2 \cup B_3)} \abs{\phi^h_{\bz}(0, x) - \phi^h_{\bz}(t, x)} \dd x \nonumber \\ & \qquad + \frac{1}{2}\int_{\R\setminus (B_2 \cup B_3)} \abs{\phi^h_{\bz}(0, x) - \phi^h_{\bz}(t, x)} \dd x. \label{eq:early stop decomp}
    \end{align}
    For the second term in \eqref{eq:early stop decomp}, Assumption \ref{ass:truncated-subGaussian} and Lemma \ref{lemma:E2-terms} imply
    \begin{align}
        \int_{\R\setminus (B_2 \cup B_3)} \abs{\phi^h_{\bz}(0, x) - \phi^h_{\bz}(t, x)} \dd x & \leq \int_{\R\setminus B_2 } \phi^h_{\bz}(0, x)  \dd x + \int_{\R\setminus B_3}  \phi^h_{\bz}(t, x) \dd x \nonumber\\ & \lesssim (R_3)^{-1}\exp(-C_2R_3^2/2) + (R_6)^{-1}\exp(-C_2'R_6^2/2) \lesssim \epsilon. \label{eq:tv-1}
    \end{align}
    as long as we take $R_3, R_6 \asymp \sqrt{\log(1/\epsilon)}$. 
    Take expectation of $\bX_0^{[k:h)}$ to have 
    \begin{align*}
        \E_{\bX_0^{[k:h)}}[{\rm TV}(\mu_0^{h}[\bX_0^{[k:h)}], \mu_{t_0}^{h}[\bX_0^{[k:h)}])] \lesssim \int \int_{(B_2 \cup B_3)} \abs{\phi^h_{\bz}(0, x) - \phi^h_{\bz}(t, x)} p^{[k:h)}( \bz) \dd x \dd \bz + \epsilon.
    \end{align*}
    We consider a region $B_4 \coloneqq \curly{\bz:\norm{\bz}_\infty \leq R_1}$. Then we have
    \begin{align}
        \int_{\R \setminus B_4} \int_{(B_2 \cup B_3)} \abs{\phi^h_{\bz}(0, x) - \phi^h_{\bz}(t, x)} p^{[k:h)}( \bz) \dd x \dd \bz \lesssim \epsilon\sqrt{\log(1/\epsilon)}, \label{eq:tv-4}
    \end{align}
    as long as we choose $R_1 \asymp \sqrt{\log(1/\epsilon)}$. 
    
    To bound the first term in \eqref{eq:early stop decomp}, we let $K_t(x, y) = \frac{1}{\sigma_t(2\pi)^{1/2}}\exp\big(-(2\sigma_t^2)^{-1}\abs{x - \alpha_t y}^2\big) $ be an auxiliary kernel. Recall $B_{1} = \bracket{\frac{x - \sigma_t R_2}{\alpha_t}, \frac{x + \sigma_t R_2}{\alpha_t}}$. Since $\int_\R K_t(x, y) \dd y = (\alpha_t)^{-1}$, we have
    \begin{align}
        \phi^h_{\bz}(0, x) - \phi^h_{\bz}(t, x) &  = \alpha_t\int_\R \phi^h_{\bz}(0, x) K_t(x, y) \dd y - \int_\R \phi^h_{\bz}(0, y) K_t(x, y) \dd y \nonumber \\ & = \int_\R (\phi^h_{\bz}(0, x) - \phi^h_{\bz}(0, y)) \alpha_t K_t(x, y) \dd y  + (\alpha_t - 1)\phi^h_{\bz}(t, x) \nonumber \\ & = \int_{B_1} (\phi^h_{\bz}(0, x) - \phi^h_{\bz}(0, y)) \alpha_t K_t(x, y) \dd y  + (\alpha_t - 1)\phi^h_{\bz}(t, x) \nonumber \\ & \qquad + \int_{\R\setminus B_1} (\phi^h_{\bz}(0, x) - \phi^h_{\bz}(0, y)) \alpha_t K_t(x, y) \dd y.  \label{eq:phi-0-diff}
    \end{align}
    Since $\phi^h_{\bz}(0, x) \leq C_1$ for any $x$ by Assumption \ref{ass:truncated-subGaussian}, we deduce that
    \begin{align*}
       \int_{\R\setminus B_1} \abs{\phi^h_{\bz}(0, x) - \phi^h_{\bz}(0, y)} \alpha_t K_t(x, y) \dd y \lesssim \int_{\abs{u} > R_2}\exp(-u^2/2)\dd u \lesssim (R_2)^{-1}\exp(-R_2^2/2). 
    \end{align*}
    We choose $R_2 \asymp \sqrt{\log(1/\epsilon)}$ to make the integral less than $\epsilon$. To handle the integral over $B_1$, we apply Assumption \ref{ass:holder-re} to obtain
    \begin{align*}
        \abs{\phi^h_{\bz}(0, x) - \phi^h_{\bz}(0, y)} \leq B\abs{x - y} \leq \frac{B}{\alpha_t}\abs{\alpha_t y - x} + \frac{B(1 - \alpha_t)}{\alpha_t}\abs{x} \leq \frac{B\sigma_t}{\alpha_t}R_2 + \frac{B(1 - \alpha_t)}{\alpha_t}\abs{x}.
    \end{align*}
    Thus, Eq.~\eqref{eq:phi-0-diff} becomes
    \begin{align*}
        \abs{\phi^h_{\bz}(0, x) - \phi^h_{\bz}(t, x)} \lesssim \frac{\sigma_t}{\alpha_t}R_2 + \frac{1-\alpha_t}{\alpha_t}\abs{x} + (\alpha_t - 1)\phi^h_{\bz}(t, x) + \epsilon.
    \end{align*}
    Then, integration $x$ over $B_2 \cup B_3$ gives us
    \begin{align}
        \int_{(B_2 \cup B_3)} \abs{\phi^h_{\bz}(0, x) - \phi^h_{\bz}(t, x)} \dd x & \lesssim \frac{\sigma_t}{\alpha_t}R_1 + \frac{1-\alpha_t}{\alpha_t}\int_{B_2 \cup B_3}\abs{x} \dd x + \int_{B_2 \cup B_3} (\alpha_t - 1)\phi^h_{\bz}(t, x) + \epsilon \dd x   \nonumber \\ & \lesssim \frac{\sigma_t}{\alpha_t}R_1 + \frac{1-\alpha_t}{\alpha_t}\int_{B_2 \cup B_3}\abs{x} \dd x + (\alpha_t - 1) + \epsilon\sqrt{\log(1/\epsilon)}. \label{eq:tv-2}
    \end{align}
    Consequently, the integration over $B_1$ with respect to the density $p^{[k:h)}(\cdot)$ leads to
    \begin{align}
        & \int_{B_4}\int_{(B_2 \cup B_3)} \abs{\phi^h_{\bz}(0, x) - \phi^h_{\bz}(t, x)} p^{[k:h)}(\bz) \dd x \dd \bz \nonumber   \\ & \lesssim \parenthesis{\frac{\sigma_t}{\alpha_t}R_1 + \frac{1-\alpha_t}{\alpha_t}\log(1/\epsilon) + (\alpha_t - 1) + \epsilon\sqrt{\log(1/\epsilon)}}(\sqrt{\log(1/\epsilon)})^{h-k},\label{eq:tv-3}
    \end{align}
    where we use the fact that $\abs{x} \lesssim \sqrt{\log(1/\epsilon)}$ over $B_4 \times (B_2\cup B_3)$ and $\abs{B_2\cup B_3} \lesssim \sqrt{\log(1/\epsilon)}$.  
    Combining \eqref{eq:tv-1}, \eqref{eq:tv-4}, and \eqref{eq:tv-3},  to obtain
    \begin{align*}
         & \E_{X_0^{[k:h)}}[{\rm TV}(\mu_0^{h}[\bX_0^{[k:h)}], \mu_{t_0}^{h}[\bX_0^{[k:h)}])] \\ & \lesssim \parenthesis{\frac{\sigma_t}{\alpha_t}\sqrt{\log(1/\epsilon)} + \frac{1-\alpha_t}{\alpha_t}\log(1/\epsilon) + (\alpha_t - 1) + \epsilon\sqrt{\log(1/\epsilon)}}(\sqrt{\log(1/\epsilon)})^{h-k} + \epsilon\sqrt{\log(1/\epsilon)} + \epsilon. 
    \end{align*} 
    Note that $\frac{\sigma_{t}}{\alpha_{t}} = \cO(\sqrt{t})$, $\frac{1-\alpha_t}{\alpha_t} = \cO(t)$, $\alpha_t - 1 = \cO(t)$ as $t \to 0$. We take $\epsilon = t_0$ to conclude
    \begin{align*}
        \E_{\bX_0^{[k:h)}}[{\rm TV}(\mu_0^{h}[\bX_0^{[k:h)}], \mu_{t_0}^{h}[\bX_0^{[k:h)}])] = \cO\parenthesis{\sqrt{t_0}(\log(1/t_0))^{\frac{d_{\rm trunc}}{2}}}.
    \end{align*}

    Finally, we apply the same analysis as in \cite{fu2024unveil} but replacing \cite[Theorem 4.1]{fu2024unveil} with Theorem \ref{thm:score est}. Thus, we obtain
    \begin{align*}
        \E_{\cD} [\E_{\bX_0^{[k:h)}}[{\rm TV}(\mu_0^{h}[\bX_0^{[k:h)}], \hat{\mu}_{t_0}^{h}[\bX_0^{[k:h)}])]] \lesssim \sqrt{t_0}(\log(1/t_0))^{\frac{h-k+1}{2}} + \exp(-T) + \sqrt{T\E_{\cD}[\cR(\hat{s})]},
    \end{align*}
    where we use the assumption that $\E_{\bX_0^{[k:h)}}[{\rm KL}(\mu_0^h[\bX_0^{[k:h)}]\| \cN(0, 1))] < \infty $. 
    Therefore, we complete the proof by applying \eqref{eq:tv-triangle}.
\end{proof}

\subsection{Example: ReLU Network}\label{sec:relu}
In this subsection, we specialize the function class to ReLU neural networks by leveraging existing results in the literature. This leads to explicit score estimation and distribution estimation guarantees.

Specifically, let $\cS^h$ be the ReLU network class defined as in \cite[Eq.~(2.7)]{fu2024unveil}. Consequently, Eq.~(A.15) in \cite{fu2024unveil} guarantees the existence of a ReLU network satisfying Assumption \ref{ass:net-class-1} and moreover, the logarithmic $\delta$-covering number is bounded with 
\begin{align}
    \log \cN(\delta) \lesssim N^{d_{\rm trunc}}\log^9 N(\log^8 N + \log^2 N\log \log N + \log (\delta^{-1})).  \label{eq:covering-relu}
\end{align}
With \eqref{eq:covering-relu}, we obtain concrete score estimation and distribution estimation results under the ReLU~parameterization.

\begin{corollary}[Score estimation with ReLU networks]\label{thm:score est-relu}
     Suppose Assumptions \ref{ass:truncated-subGaussian} and \ref{ass:holder-re} hold. By taking the network size parameter $N = n^{\frac{1}{d_{\rm trunc} + \beta}}$,  early stopping threshold $t_0 = N^{-C_\sigma} < 1$ and terminal horizon $T = C_\alpha\log n$, then it holds that
    \begin{align*}
        \E_{\cD}[\cR(\hat{s})] = \cO\Big(\frac{1}{t_0}n^{-\frac{\beta}{d_{\rm trunc} + \beta}} (\log n)^{\max\curly{2 + \beta/2, 18}}\Big).
    \end{align*}
\end{corollary}

\begin{proof}
We substitute \eqref{eq:covering-relu} into Theorem \ref{thm:score est} and choose $N$, $t_0$ and $T$ as stated to obtain the result.

\end{proof}
    
\begin{corollary}
[Distribution estimation with ReLU networks]\label{thm:dist est-relu} Let $h \in [H]$ and $k \in [h-1]$ be given indices.  Assume that $\E_{\bX_0^{[k:h)}}[{\rm KL}(\mu_0^h[\bX_0^{[k:h)}]\| \cN(0, 1))] < \infty $. Under Assumptions \ref{assumption:dependence_decay} \ref{ass:truncated-subGaussian}, and \ref{ass:holder-re}, by taking the early-stopping time $t_0 = n^{-\frac{\beta}{2(d_{\rm trunc} + \beta)}}$ and terminal time $T = \frac{\beta}{4(d_{\rm trunc} + \beta)}\log n$, it holds that, 
\begin{align*}
    \E_{\cD} [\E_{\bX_0^{[1:h)}}[{\rm TV}(\mu_0^{h}[\bX_0^{[1:h)}], \hat{\mu}_{t_0}^{h}[\bX_0^{[1:h)}])]]  = \cO\parenthesis{n^{-\frac{\beta}{4(d_{\rm trunc} + \beta)}}(\log n)^{\max\curly{ d_{\rm trunc}, 3 + \beta/2, 19}/2} + u(d_{\rm trunc})}.
\end{align*}
\end{corollary}

\begin{proof}
We combine Theorem \ref{thm:dist est} and Corollary \ref{thm:score est-relu} to have
\begin{align*}
     & \E_{\cD} [\E_{\bX_0^{[1:h)}}[{\rm TV}(\mu_0^{h}[\bX_0^{[1:h)}], \hat{\mu}_{t_0}^{h}[\bX_0^{[1:h)}])]] \\ & \lesssim \sqrt{t_0}(\log(1/t_0))^{\frac{d_{\rm trunc}}{2}} + \exp(-T) + \sqrt{\frac{T}{t_0}n^{-\frac{\beta}{d_{\rm trunc} + \beta}} (\log n)^{\max\curly{2 + \beta/2, 18}}} + u(d_{\rm trunc}).
\end{align*}
We take $T = \frac{\beta}{4(d_{\rm trunc} + \beta)}\log n$, $t_0 = n^{-\frac{\beta}{2(d_{\rm trunc} + \beta)}}$ to deduce that
    \begin{align*}
         & \E_{\cD} [\E_{\bX_0^{[1:h)}}[{\rm TV}(\mu_0^{h}[\bX_0^{[1:h)}], \hat{\mu}_{t_0}^{h}[\bX_0^{[1:h)}])]]   \\ & \lesssim n^{-\frac{\beta}{4(d_{\rm trunc} + \beta)}}(\log n)^{\frac{d_{\rm trunc}}{2}} + n^{-\frac{\beta}{4(d_{\rm trunc} + \beta)}} + \sqrt{\log n}\Big(\frac{1}{t_0}n^{-\frac{\beta}{d_{\rm trunc}+\beta}} (\log n)^{\max\curly{2 + \beta/2, 18}}\Big)^{1/2}  + u(d_{\rm trunc}) \\ & = \cO\parenthesis{n^{-\frac{\beta}{4(d_{\rm trunc} + \beta)}}(\log n)^{\max\curly{ d_{\rm trunc}, 3 + \beta/2, 19}/2} + u(d_{\rm trunc})}.
    \end{align*}
    Therefore, we complete the proof.

\end{proof}

As an intermediate consequence, we establish a joint distribution estimation result following from Corollary~\ref{thm:dist est-relu} and the chain rule \eqref{eq:tv-chain}.
\begin{corollary}
    Suppose the assumptions in Corollary \ref{thm:dist est-relu} hold for all $h \in [H]$, it holds that
    \begin{align*}
         \E_{\cD}[{\rm TV}(\mu_0^{[1:H]}, \hat{\mu}_{t_0}^{[1:H]})] = \cO\parenthesis{\sum_{h = 1}^H \Big[n^{-\frac{\beta}{4(d_{\rm trunc} + \beta)}}(\log n)^{\max\curly{ d_{\rm trunc}, 3 + \beta/2, 19}/2} + u(d_{\rm trunc})\Big] },
    \end{align*}
    where $d_{\rm trunc} = h - k(h) + 1$ and $k(h)$ is the optimal index chosen for each $h$. 
\end{corollary}

\paragraph{Selection of $d_{\rm trunc}$.}
We now discuss how to choose the truncation length $d_{\rm trunc}$ to balance the two competing terms in Corollary~\ref{thm:dist est-relu}: the statistical learning error, which grows with $d_{\rm trunc}$, and the truncation error $u(d_{\rm trunc})$, which decreases with $d_{\rm trunc}$. We state the resulting rates for two canonical decay regimes.

\begin{proposition}[Exponential decay]\label{cor:exp-decay}
Suppose there are constants $C_u, \kappa > 0$ such that
\[
u(d) \leq C_u\parenthesis{e^{-\kappa d} - e^{-\kappa h}},
\qquad 1 \leq d \leq h.
\]
If we  take
\begin{align}\label{eq:dtrunc-exp}
    d_{\rm trunc}
    &=
    \min\left\{
    h,
    \left\lceil
    \Delta_n
    \right\rceil
    \right\}, \; \text{with} \;\Delta_n \coloneqq \frac{-\beta+\sqrt{\beta^2+(\beta/\kappa)\log n}}{2},
\end{align}
then under the assumptions of Corollary~\ref{thm:dist est-relu}, it holds that
\begin{align*}
    \E_{\cD} \left[\E_{\bX_0^{[1:h)}}\left[{\rm TV}\left(\mu_0^{h}[\bX_0^{[1:h)}], \hat{\mu}_{t_0}^{h}[\bX_0^{[1:h)}]\right)\right]\right]  = 
    \begin{cases}
        \cO\parenthesis{(\log n)^{\max\curly{d_{\rm trunc}, 3 + \beta/2, 19}/2}e^{-\kappa d_{\rm trunc}}}, & h>\Delta_n  \\ 
        \cO\parenthesis{n^{-\frac{\beta}{4(h + \beta)}}(\log n)^{\max\curly{ h, 3 + \beta/2, 19}/2}}, & h \leq \Delta_n.
    \end{cases}
\end{align*}
\end{proposition}

{ For a fixed horizon $h$, Proposition~\ref{cor:exp-decay} provides a truncation rule when $u(\cdot)$ follows an exponential decay. For timestamps such that $h>\Delta_n$, the rule selects a truncated history length of order $\Theta(\sqrt{\log n})$. Since the information from a small dataset is limited, truncating to the most recent timestamps reduces the input dimension and thus improving the statistical learning error. As $n$ grows, the rule increases $d_{\rm trunc}$ until $d_{\rm trunc}=h$, at which point the bound reduces to the statistical learning error with dimension $h$. }

\begin{proof}
By definition, $\Delta_n$ satisfies the quadratic equation
\begin{align}
    \frac{\beta\log n}{4(\Delta_n+\beta)} = \kappa \Delta_n. \label{eq:quadratic}
\end{align}
Taking exponentials on both sides gives us
\begin{align*}
    n^{-\frac{\beta}{4(\Delta_n+\beta)}} = e^{-\kappa \Delta_n}.
\end{align*}
When $d_{\rm trunc} < h$, we have $d_{\rm trunc} = \left\lceil\Delta_n\right\rceil$ and thus $ \Delta_n \leq d_{\rm trunc} \leq \Delta_n + 1 $. Consequently, the following bounds hold
\begin{align}
    e^{-\kappa d_{\rm trunc}}\leq e^{-\kappa \Delta_n} = n^{-\frac{\beta}{4(\Delta_n+\beta)}} \leq n^{-\frac{\beta}{4(d_{\rm trunc}+\beta)}},\label{eq:exp-lower}
\end{align}
and \eqref{eq:quadratic} implies
\begin{align}
    n^{-\frac{\beta}{4(d_{\rm trunc}+\beta)}} & \leq n^{-\frac{\beta}{4(\Delta_n + 1 +\beta)}} = \exp\curly{-\frac{\beta \log n}{4(\Delta_n + 1 + \beta)}} = e^{-\kappa d_{\rm trunc}}\exp\curly{\kappa d_{\rm trunc}-\kappa \Delta_n\frac{\Delta_n + \beta}{\Delta_n + 1 + \beta}} \nonumber \\ & \leq e^{-\kappa d_{\rm trunc}}\exp\curly{\kappa (\Delta_n + 1)-\kappa \Delta_n\frac{\Delta_n + \beta}{\Delta_n + 1 + \beta}} = e^{-\kappa d_{\rm trunc}}\exp\curly{\kappa\parenthesis{1 + \frac{\Delta_n}{\Delta_n + 1 + \beta}}} 
    \nonumber\\ & \leq e^{2\kappa}e^{-\kappa d_{\rm trunc}}. \label{eq:exp-upper}
\end{align}
Moreover, it holds that
\begin{align}
    e^{-\kappa d_{\rm trunc}} \geq e^{-\kappa d_{\rm trunc}} - e^{-\kappa h} = e^{-\kappa d_{\rm trunc}}(1 - e^{-\kappa(h-d_{\rm trunc})}) \geq (1 - e^{-\kappa})e^{-\kappa d_{\rm trunc}}.  \label{eq:exp}
\end{align}
Therefore, we combine \eqref{eq:exp-lower}, \eqref{eq:exp-upper} and \eqref{eq:exp} to conclude that 
\begin{align*}
     e^{-\kappa d_{\rm trunc}} - e^{-\kappa h} \leq n^{-\frac{\beta}{4(d_{\rm trunc}+\beta)}} \leq \frac{e^{2\kappa}}{1 - e^{-\kappa}} (e^{-\kappa d_{\rm trunc}} - e^{-\kappa h}).
\end{align*}
This completes the proof together with $ d_{\rm trunc} = h $ case.

\end{proof}

Invoke the Lambert W function $W:(0,\infty)\to(0,\infty)$ as the inverse of the mapping $x\mapsto xe^x$, i.e., $W(z)e^{W(z)}=z$ for all $z>0$. We have the following result for the polynomial decay case.

\begin{proposition}[Polynomial decay]\label{cor:poly-decay}
Suppose there are constants $C_u, \kappa > 0$ such that
\[
u(d) \leq C_u\parenthesis{d^{-\kappa} - h^{-\kappa}},
\qquad 1 \leq d \leq h.
\]
If we take
\begin{align}\label{eq:dtrunc-poly}
    d_{\rm trunc} = \min\curly{h, \left\lceil \Delta_n' \right\rceil}, \;\text{with}\; \Delta_n' \coloneqq \frac{{\beta \log n}/{4\kappa}}{W({\beta \log n}/{4\kappa})},
\end{align}
then under the assumptions of Corollary~\ref{thm:dist est-relu}, it holds that
\begin{align*}
    \E_{\cD} \left[\E_{\bX_0^{[1:h)}}\left[{\rm TV}\left(\mu_0^{h}[\bX_0^{[1:h)}], \hat{\mu}_{t_0}^{h}[\bX_0^{[1:h)}]\right)\right]\right] = 
    \begin{cases}
        \cO\parenthesis{(\log n)^{\max\curly{d_{\rm trunc}, 3 + \beta/2, 19}/2}\parenthesis{d_{\rm trunc}}^{-\kappa}}, & h>\Delta_n' \\ 
        \cO\parenthesis{n^{-\frac{\beta}{4(h + \beta)}}(\log n)^{\max\curly{ h, 3 + \beta/2, 19}/2}}, & h\leq \Delta_n'.
    \end{cases}
\end{align*}
\end{proposition}

{ For polynomial decay, it holds that
$W(\beta \log n/4\kappa)=\Theta(\log\log n)$ in \eqref{eq:dtrunc-poly}. Consequently, the truncation rule selects a history length of order
$\Theta(\log n/\log\log n)$ before it reaches the horizon $h$.
Compared with the exponential-decay case, polynomial decay requires a larger truncation window, since the impact from the remote history decays more slowly.}

\begin{proof}
Denote $A_n \coloneqq \beta\log n /4\kappa$. Then we have $\Delta_n'=A_n/W(A_n)$.
The definition of the function $W$ gives
\begin{align}\label{eq:poly-W-balance}
    \Delta_n' = \frac{W(A_n)e^{W(A_n)}}{W(A_n)} = e^{W(A_n)}, \;\text{and} \;\Delta_n' \log \Delta_n' = e^{W(A_n)}W(A_n) = A_n.
\end{align}
Consider the case when $d_{\rm trunc} < h$. Then we have $ d_{\rm trunc} = \lceil \Delta_n' \rceil \leq \Delta_n' + 1 $. It follows from \eqref{eq:poly-W-balance} that
\begin{align}
    n^{-\frac{\beta}{4(d_{\rm trunc}+\beta)}}
    &= \exp\curly{-\kappa\frac{A_n}{d_{\rm trunc}+\beta}} 
    \leq \exp\curly{-\kappa\frac{A_n}{\Delta_n'+1+\beta}} \nonumber\\
    &= \exp\curly{-\kappa\frac{\Delta_n'\log \Delta_n'}{\Delta_n'+1+\beta}}
    = (\Delta_n')^{-\kappa}\exp\curly{\kappa\frac{(1+\beta)\log \Delta_n'}{\Delta_n'+1+\beta}} \nonumber\\
    &\leq \exp\curly{\frac{\kappa(1 + \beta)}{e}} (\Delta_n')^{-\kappa}, \label
    {eq:poly-stat-bound}
\end{align}
where we use the fact that $ \frac{\log x}{x + 1 + \beta} \leq \frac{\log x}{x} \leq 1/e $ for all $x \geq 1 $. Moreover, since $d_{\rm trunc}=\left\lceil \Delta_n'\right\rceil\geq \Delta_n'$, we have
\begin{align}
    d_{\rm trunc}^{-\kappa}-h^{-\kappa} \leq d_{\rm trunc}^{-\kappa}\leq (\Delta_n')^{-\kappa}. \label{eq:poly-trunc-bound}
\end{align}
This completes the proof together with $ d_{\rm trunc} = h $ case. 

\end{proof}

\section{Empirical Analysis}
In this section, we numerically evaluate the performance of our \mbox{\method{}~framework} proposed in Section \ref{sec:model} through extensive experiments. We first test the model on synthetic data generated by ARMA models and Gaussian processes. Specifically, we evaluate its effectiveness numerically in terms of recovering both the ground-truth distribution and the underlying temporal structures. Next, we apply our \mbox{\method{}~framework} to real-world data and evaluate its performance in constructing mean-variance optimal portfolios. %
{ In implementation, we use a single transformer to parameterize the score functions across all timestamps. The causal mask ensures that the score at each timestamp depends only on the past history, while the objective \eqref{eq:dsm-final} allows us to train the coordinatewise scores {\it in parallel}.}

\subsection{Synthetic Data}
In this subsection, we evaluate the \mbox{\method{}~framework} for learning ARMA models and Gaussian processes, two benchmark synthetic data generators.

\paragraph{ARMA models.}
We first train an \mbox{\method{}~model} on data generated by an ARMA(2, 2)  with autoregressive coefficients  $(\varphi_1, \varphi_2) = (0.9, -0.01)$ and moving average coefficients $(\theta_1, \theta_2) = (0.1, 0.1)$. The diffusion models are trained with a fixed sequence length $ H = 500$ and a variety of sample sizes $n \in \curly{2^8, 2^9, 2^{10}, 2^{11}}$. For each trained model, we generate 1024 samples for evaluation.

We assess the accuracy of the learned temporal dependence through the lens of autocorrelation function (ACF). Let $\lambda_i$ denote the ground-truth ACF of the ARMA model, and $\curly{\lambda_i^{\rm Diff}}_{i=1}^K$ and $\curly{\lambda_i^{\rm Emp}}_{i=1}^K$ denote the top-$K$ ACF lags  calculated from diffusion-generated samples and empirical training samples, respectively. We introduce and then compare the following $\ell_1$ relative errors in ACF between the diffusion model generated and the empirical training samples,
\begin{align*}
    {\rm Diff~RE}_1 \coloneqq \frac{\sum_{i=1}^K\abs{\lambda_i^{\rm Diff} - \lambda_i}}{\sum_{i=1}^K\abs{ \lambda_i}}, \quad {\rm Emp~RE}_1 \coloneqq \frac{\sum_{i=1}^K\abs{\lambda_i^{\rm Emp} - \lambda_i}}{\sum_{i=1}^K\abs{ \lambda_i}}.
\end{align*}

Table \ref{tab: arma} reports the errors in estimating top-40 lags of an ARMA model, with each experiment repeated five times to ensure robustness.  Figure \ref{fig:arma size} highlights the advantage of diffusion-generated samples in estimating ACFs across different sample sizes. Notably, as shown in Figure \ref{fig:arma size} (b), the diffusion samples (green) perfectly recover ground-truth ACFs (blue) for lags beyond 30 while empirical training samples recover the long-lag ACFs less accurately.  This capacity is particularly important for financial applications as high-quality historical data is often scarce, making precise ACF estimation a significant challenge.

\begin{figure}[t]
    \centering
    \begin{subfigure}[b]{0.46\textwidth}
        \centering
        \includegraphics[width=\textwidth]{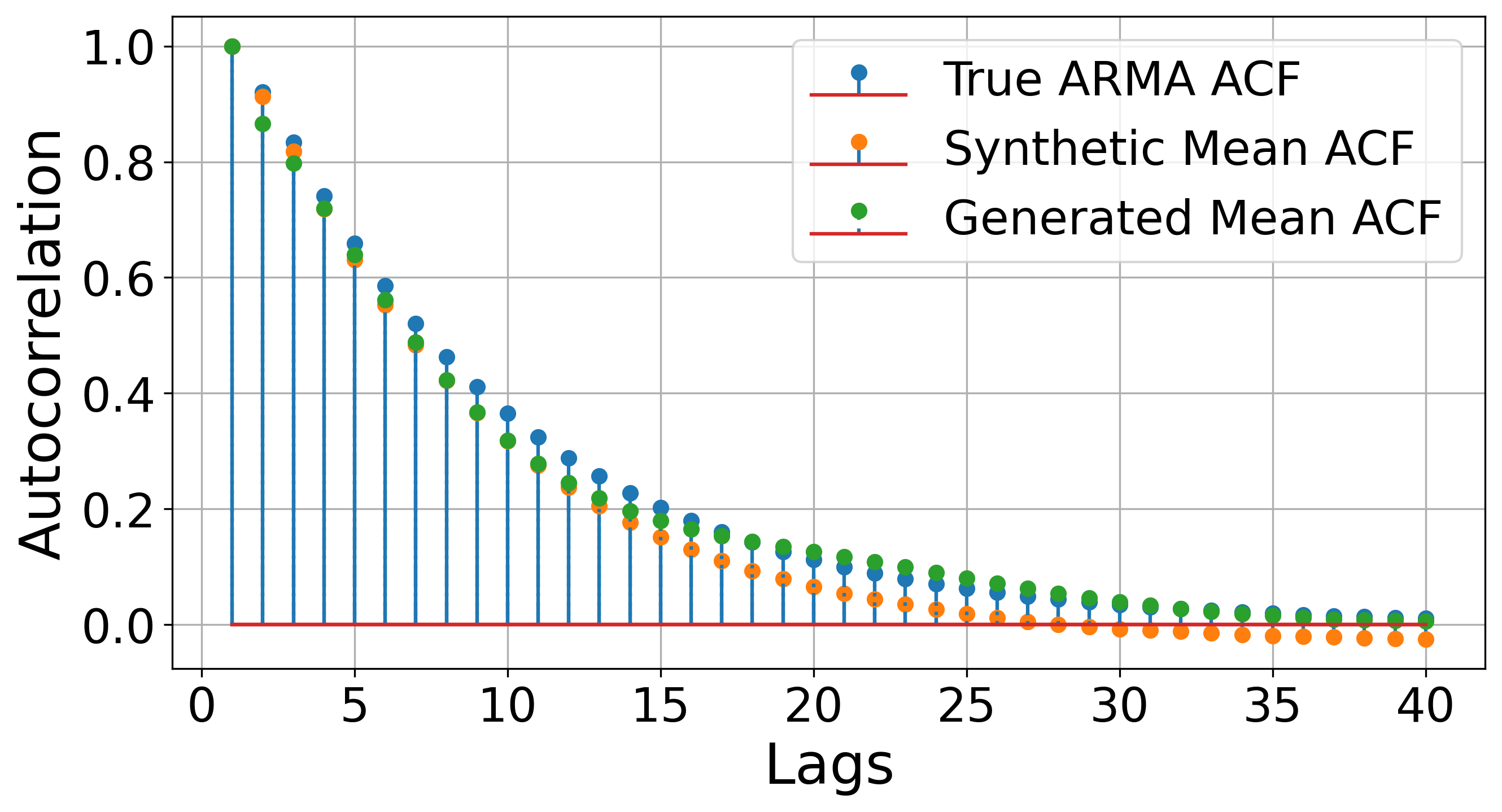}
        \caption{$n = 256$}
        \label{fig:sub1}
    \end{subfigure}
    \hfill
    \begin{subfigure}[b]{0.46\textwidth}
        \centering
        \includegraphics[width=\textwidth]{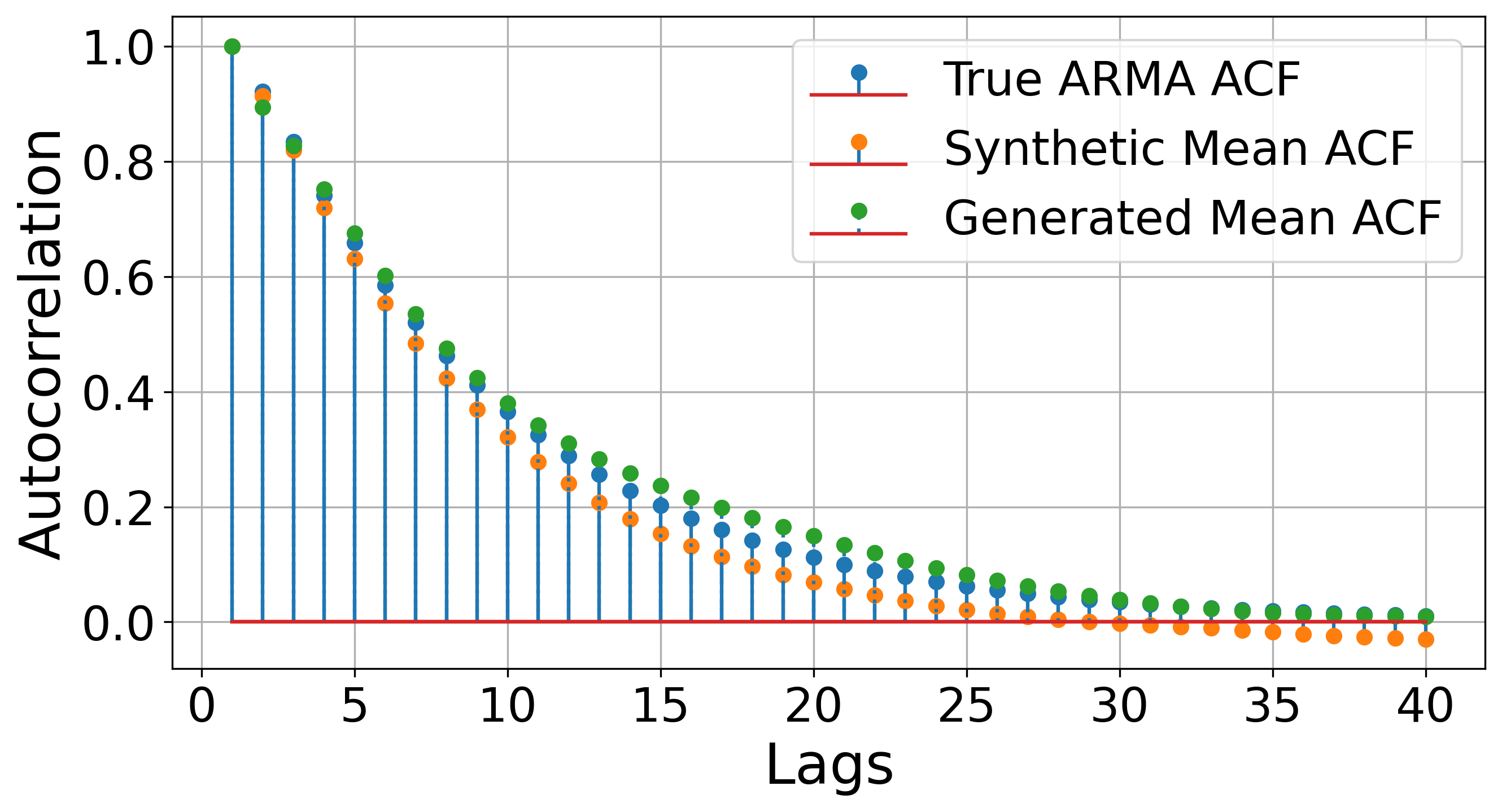}
        \caption{$n = 512$}
        \label{fig:sub2}
    \end{subfigure}

    \vspace{-2pt}
    \begin{subfigure}[b]{0.46\textwidth}
        \centering
        \includegraphics[width=\textwidth]{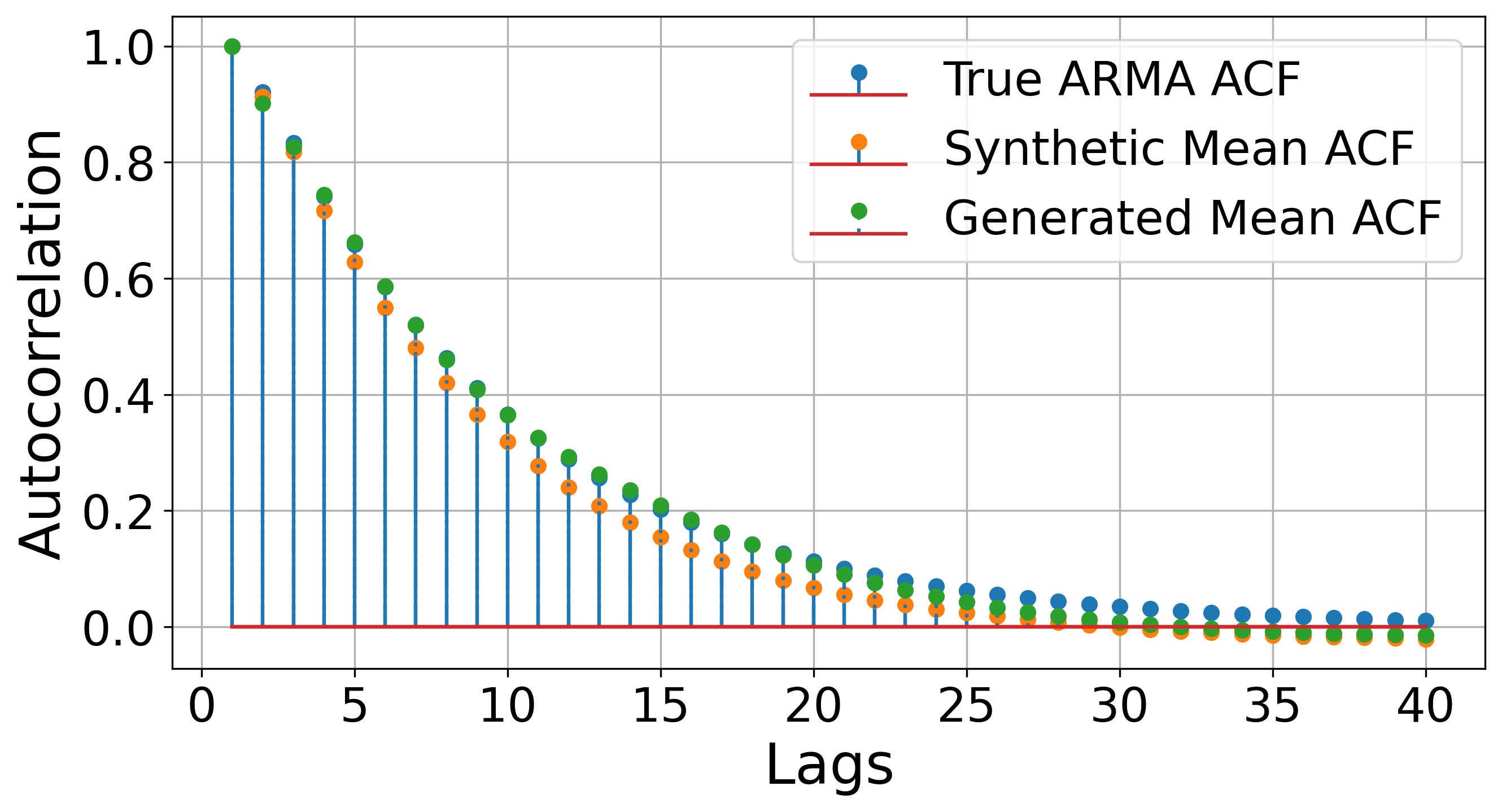}
        \caption{$n = 1024$}
        \label{fig:sub3}
    \end{subfigure}
    \hfill
    \begin{subfigure}[b]{0.46\textwidth}
        \centering
        \includegraphics[width=\textwidth]{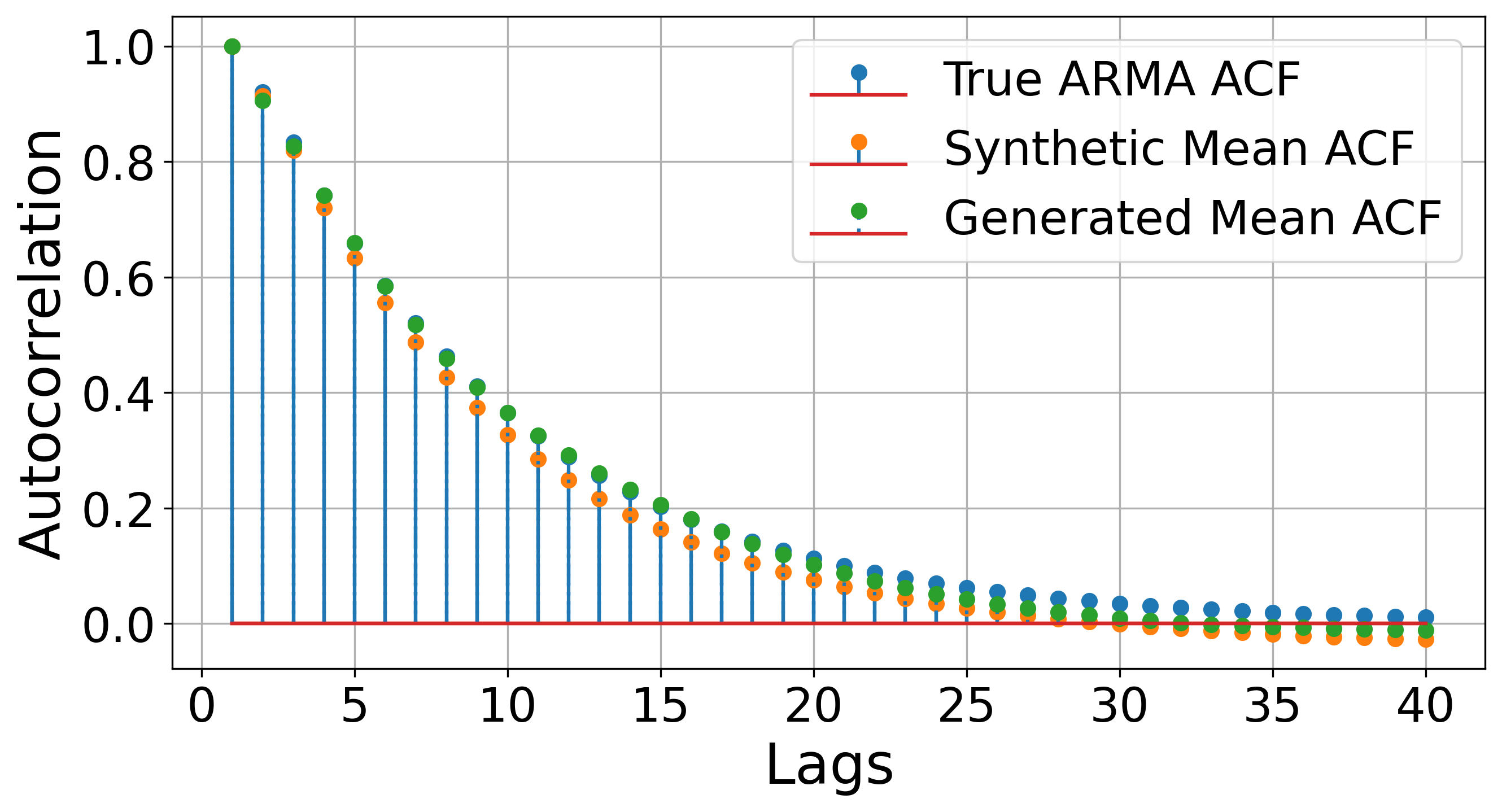}
        \caption{$n = 2048$}
        \label{fig:sub4}
    \end{subfigure}

    \vspace{-2pt}
    \caption{Estimated ACFs with varying training dataset size.}
    \label{fig:arma size}
\end{figure}

\begin{table}[htbp]
    \centering
    \small
    {
    \begin{tabularx}{0.8\textwidth}{l *{3}{>{\centering\arraybackslash}X}}
    \toprule
    \multicolumn{3}{c}{\textbf{ARMA Models with Varying Train Dataset Size ($H = 500$)}} \\
    \midrule
    $n$ & \textbf{Diff RE$_1$} & \textbf{Emp RE$_1$} \\
    \midrule
    $2^{8}=256$  & 0.1171 ($\pm$ 0.0266) & 0.1431 ($\pm$ 0.0309)  \\
    $2^{9}=512$  & 0.1308 ($\pm$ 0.0484) & 0.1528 ($\pm$ 0.0123) \\
    $2^{10}=1024$ & 0.1074 ($\pm$ 0.0373) & 0.1503 ($\pm$ 0.0067) \\
    $2^{11}=2048$ & 0.1172 ($\pm$ 0.1446) & 0.1487 ($\pm$ 0.0045) \\
    \bottomrule
    \end{tabularx}
    }
    \caption{Relative error of the estimated top-40 ACFs of ARMA models for varying training dataset size (standard deviations in parentheses).}
    \label{tab: arma}
\end{table}

\paragraph{Gaussian processes.}
Next, we evaluate the performance of \mbox{\method{}~model} for synthetic Gaussian process (GP)
sequence %
of length $H = 32$. Specifically,  following the experimental setup in \cite{fu2024diffusion}, we define the target distribution as $\bX_0 = (X_0^1, \dots, X_0^H)^\top \sim \cN(\bmu, \bSigma)$. The mean vector is set to zero, $\bmu = [\mu_1, \dots, \mu_H]^\top = \bzero$ and the covariance matrix is defined as $\bSigma = \sigma^2\bGamma$ where $\sigma^2 > 0$ and $\bGamma \in \R^{H \times H}$. The structure of the covariance is determined by the kernel function $\gamma(\cdot)$ such that the ($i, j$)-th entry of $\bGamma$ is $\Gamma_{ij} = \gamma(i, j)$. Similar to \cite{fu2024diffusion}, an exponential kernel function is chosen, i.e.,  $\gamma(i, j) = \exp(-\abs{i - j}^\nu/\ell)$, with hyperparameters $\ell = 4$ and $\nu = 0.8$. 

To evaluate the model, we use $n = 2048$ samples from the ground-truth GP for training. After training, we generate another batch of $2048$ samples to estimate the empirical covariance. As illustrated in Figure \ref{fig:GP}, the covariance structure of the diffusion-generated samples successfully recovers the characteristic decay pattern of the ground truth. This alignment demonstrates that our sequential diffusion models are capable of accurately capturing the joint distribution and the underlying temporal correlations of the sequential data.

\begin{figure}[htbp]
    \centering
    \includegraphics[width=\linewidth]{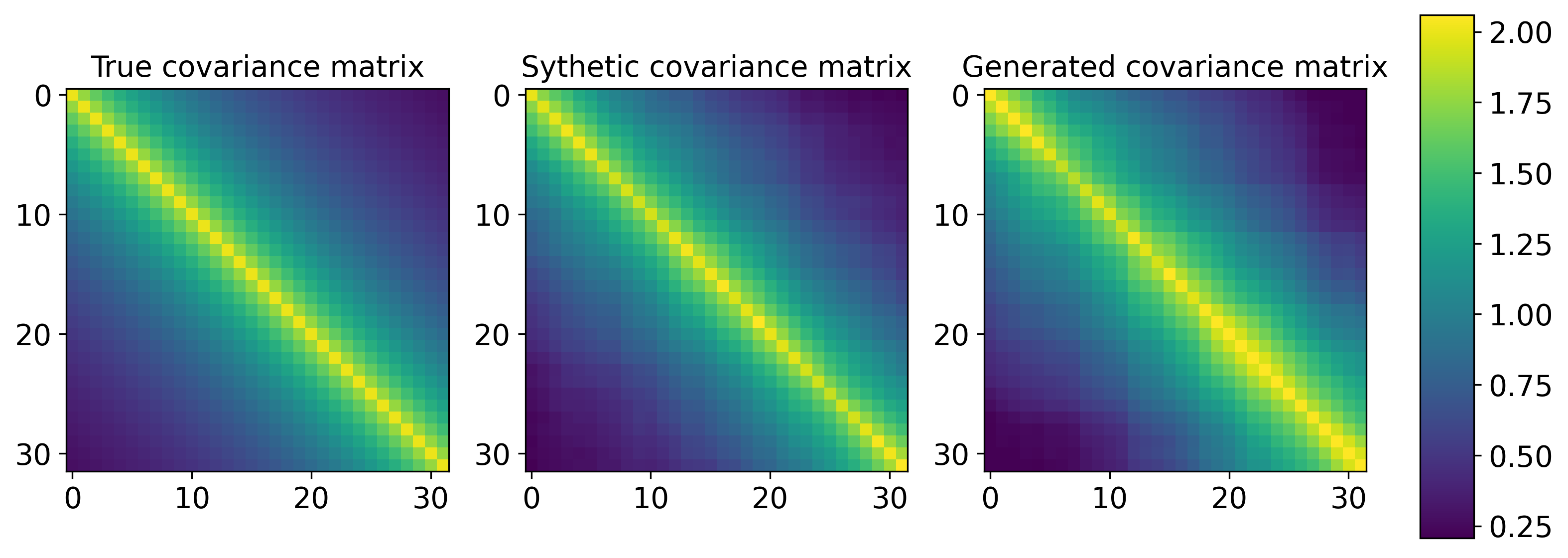}
    \caption{Left: ground-truth covariance matrix used for training data generation; Middle: covariance of training dataset; Right: covariance of samples generated by our diffusion models. Sequential diffusion model learns spatio-temporal~dependence.}
    \label{fig:GP}
\end{figure}

\subsection{Mean-Variance Portfolio Optimization} \label{sec:mean-var}
In this subsection, we study a mean-variance portfolio optimization problem which was studied by \citet{gao2025data} and \citet{wang2020continuous}. Specifically, an agent invests
in the S\&P 500 index and a risk-free asset with interest rate $r = 2\%$ for half a year with $H = 128$ trading days. The problem is to find a self-financing strategy $\btheta = \curly{\theta^h}_{h=1}^H$ where $\theta^h \in [0, 1]$ denotes the proportion of wealth invested in the risky asset at time $h$ %
that minimizes the variance of terminal wealth $X^{H, \theta}$ for a given target wealth level $z$, i.e.,
\begin{align}
    \min_{\btheta}\;\; {\rm Var}(X^{H, \theta}) \;\; {\rm s.t.}\;\; \E[X^{H, \theta}] = z. \label{eq:mean-var}
\end{align}

Following \cite{gao2025data}, we focus on two approaches:
\begin{enumerate}[1.]
    \item  \textbf{Plug-in policy:} A classical method that estimates coefficients of a GBM model and plugs in the estimators into the expression of the optimal deterministic policy of \eqref{eq:mean-var}. 
    \item \textbf{RL policy:} A model-free reinforcement learning (RL) approach proposed in \cite{jia2023q} which requires a market data simulator. 
\end{enumerate}
Furthermore, we consider four data sources for obtaining each policy:
\begin{enumerate}[1.]
    \item \textbf{Split:} We split daily returns of S\&P 500 data from 1990-2009 into 40 half-year trajectories (called split paths). These paths are used to train our sequential diffusion models.
    \item \textbf{Split + Synthetic:} In addition to 40 split paths, we augment 40 synthetic paths generated by the approach in \cite{gao2025data} and develop investment policies.
    \item \textbf{Split + Ours:} Compared to \textbf{Split + Synthetic}, we replace the synthetic paths with 40 trajectories generated by our sequential diffusion model. 
    \item \textbf{Bootstrap:} For RL policies, we also construct investment polices using 40 bootstrap paths as in \cite{jia2023q}.
\end{enumerate}

\paragraph{Main results.} We test all the above policies on S\&P 500 daily return data from 2010 to 2019 with bootstrap. As shown in Table \ref{table: qL_mv}, including diffusion-generated samples do not improve plug-in policies. This is aligned with the observations in \cite{gao2025data} as the data paths do not necessarily satisfy the GBM assumption. On the other hand, similar to \textbf{Split + Synthetic}, our diffusion-generated samples also significantly improve the performance of RL policies. Notably, our approach achieves the highest Sharpe ratio and outperforms the method in \cite{gao2025data}. This observation consolidates the belief that our diffusion models learn the underlying temporal structure of financial data in an adapted sequential manner. We also remark that the S\&P 500 daily return has strongly Markovian behaviors with only short-range dependence on the past. This explains why our method performs comparably to \cite{gao2025data} in this setting. However, for financial time series exhibiting longer historical dependence and more complex temporal structure, we expect our adapted sequential diffusion approach will demonstrate clearer advantages.

\begin{table}[H]
    \centering
    \small{
    \begin{tabular}{c | c | l || c c c}
        \toprule
        Target Level & Policy & Training Paths & Mean & Variance & Sharpe Ratio $\uparrow$ \\
        \hline
        \multirow{6}{*}{$z = 1.10$} & \multirow{3}{*}{Plug-in} & Split & 1.1759 ($\pm$ 0.0025) & 0.1045 ($\pm$ 0.0024) & 0.5442 ($\pm$ 0.0081)\\
        && Split + Synthetic & 1.1138 ($\pm$ 0.0037) & 0.0399 ($\pm$ 0.0025) & 0.5702 ($\pm$ 0.0083)\\ 
        && Split + Ours & 1.1685 ($\pm$ 0.0024) & 0.0952 ($\pm$ 0.0022) & 0.5462 ($\pm$ 0.0081)\\
        \cline{2-6}
        &\multirow{4}{*}{RL}  & Split & 1.0932 ($\pm$ 0.0016) & 0.0113 ($\pm$ 0.0011) & 0.8779 ($\pm$ 0.0394)\\
        && Bootstrap  & 1.0934 ($\pm$ 0.0010) & 0.0105 ($\pm$ 0.0010) & 0.9121 ($\pm$ 0.0450) \\
        && Split + Synthetic & 1.0853 ($\pm$ 0.0014) & 0.0074 ($\pm$ 0.0010) & 0.9930 ($\pm$ 0.0586) \\
        && Split + Ours & 1.0732 ($\pm$ 0.0008) & 0.0054 ($\pm$ 0.0006) & \textbf{0.9968} ($\pm$ 0.0589) \\
        \hline
        \multirow{7}{*}{$z = 1.20$} & \multirow{3}{*}{Plug-in} & Split & 1.3517 ($\pm$ 0.0050) & 0.4179 ($\pm$ 0.0094) & 0.5442 ($\pm$ 0.0081)\\
        && Split + Synthetic & 1.2262 ($\pm$ 0.0066) & 0.1572 ($\pm$ 0.0097) & 0.5708 ($\pm$ 0.0087)\\
        && Split + Ours & 1.3370 ($\pm$ 0.0048) & 0.3807 ($\pm$ 0.0087) & 0.5462 ($\pm$ 0.0081)\\
        \cline{2-6}
        &\multirow{4}{*}{RL} & Split & 1.1861 ($\pm$ 0.0032) & 0.0460 ($\pm$ 0.0045) & 0.8695 ($\pm$ 0.0376)\\
        && Bootstrap & 1.1865 ($\pm$ 0.0016) & 0.0425 ($\pm$ 0.0043) & 0.9068 ($\pm$ 0.0435)  \\
        && Split + Synthetic & 1.1706 ($\pm$ 0.0028) & 0.0302 ($\pm$ 0.0040) & 0.9866 ($\pm$ 0.0574)\\
        && Split + Ours & 1.1464 ($\pm$ 0.0016) & 0.0218 ($\pm$ 0.0025) & \textbf{0.9965} ($\pm$ 0.0591)\\
        \bottomrule
    \end{tabular}}
    \caption{Performance of different portfolio strategies  with and without synthetic paths (standard deviations in parentheses).}
    \label{table: qL_mv}
\end{table}

\paragraph{Acknowledgment.} RX was partially supported by an NSF CAREER Award DMS-2614933 and a gift fund from Point72.

\bibliographystyle{unsrtnat} %
\IfFileExists{diffusion_ot.bib}{%
  \bibliography{diffusion_ot}%
}{%
  \bibliography{../diffusion_ot}%
} %

\newpage
\appendix
\section{Omitted Proofs in Section \ref{sec:main results}}
\label{sec:proofs}

We devote this section to a few omitted proofs in Section \ref{sec:main results}.

\begin{proposition}[Multivariate Gaussian]\label{lemma:gaussian-truncated-subG}
    If $\bX_0\sim\cN(\bmu,\bSigma)$ with positive definite covariance matrix $\bSigma$, then Assumption \ref{ass:truncated-subGaussian} holds.
\end{proposition}

\begin{proof}
The marginal density of $X_0^1$ is Gaussian and thus we focus on $h > 1$ cases. Let $h\in\{2,\ldots,H\}$ and $k\in\{1,\ldots,h-1\}$ be fixed. Let $\bmu^{[k:h]} \in \mathbb{R}^{h-k+1}$ be the mean of $\bX_0^{[k:h]}$ and $\bSigma^{[k:h]} \in \R^{(h-k+1) \times (h-k+1)}$ be the covariance matrix of $\bX_0^{[k:h]}$. We further write 
\begin{align}
    \bSigma^{[k:h]} = 
    \begin{pmatrix}
        \bSigma_{11} & \bSigma_{12} \\
        \bSigma_{21} & \bSigma_{22}
    \end{pmatrix}
\end{align}
where $ \bSigma_{11} \in \R^{(h-k) \times (h-k)} $ is the covariance matrix of $\bX_0^{[k:h-1]}$, $ \bSigma_{22} \in \R $ is the variance of $X_0^{h}$, $\bSigma_{12} \in \R^{h-k} $ is the covariance between $\bX_0^{[k:h-1]}$ and $X_0^{h}$ and $ \bSigma_{21} = \bSigma_{12}^\top $. Since $\bSigma$ is positive definite, the principal submatrix $\bSigma^{[k:h]}$ is positive definite. It follows that
\begin{align*}
    X_0^h \mid \bX_0^{[k:h-1]}=\bz \sim \cN\parenthesis{\mu^h + \bSigma_{21}\bSigma_{11}^{-1}(\bz-\bmu^{[k:h-1]}), \bSigma_{22} - \bSigma_{21}\bSigma_{11}^{-1}\bSigma_{12}}.
\end{align*}
Consequently, the conditional  probability density satisfies
\begin{align*}
    \phi_{\bz}^h(0,x)
    &=
    \frac{1}{(2\pi(\bSigma_{22} - \bSigma_{21}\bSigma_{11}^{-1}\bSigma_{12}))^{1/2}}\exp\parenthesis{-\frac{\abs{x-\mu^h - \bSigma_{21}\bSigma_{11}^{-1}(\bz-\bmu^{[k:h-1]})}^2}{2(\bSigma_{22} - \bSigma_{21}\bSigma_{11}^{-1}\bSigma_{12})}} \\
    &\leq
    \frac{\exp\parenthesis{\abs{\mu^h - \bSigma_{21}\bSigma_{11}^{-1}\bmu^{[k:h-1]}}^2/(2(\bSigma_{22} - \bSigma_{21}\bSigma_{11}^{-1}\bSigma_{12}))}}{(2\pi(\bSigma_{22} - \bSigma_{21}\bSigma_{11}^{-1}\bSigma_{12}))^{1/2}}\exp\parenthesis{-\frac{\abs{x-\bSigma_{21}\bSigma_{11}^{-1}\bz}^2}{4(\bSigma_{22} - \bSigma_{21}\bSigma_{11}^{-1}\bSigma_{12})}}.
\end{align*}
This completes the proof. 
\end{proof}

\begin{lemma}\label{lemma:tail-density-t}
    Under Assumption \ref{ass:truncated-subGaussian}, the following inequalities hold:
    \begin{align}
        \phi^h_{\bz}(t, x) \leq \frac{C_1}{\parenthesis{C_2\sigma_t^2 + \alpha_t^2}^{1/2}}\exp\parenthesis{-\frac{C_2}{2(C_2\sigma_t^2 + \alpha_t^2)}(x - \alpha_t \bv^\top \bz)^2}, \label{eq:density-t-history-upper}
    \end{align}
    and
    \begin{align}
        \phi^h_{\bz}(t, x) \gtrsim \frac{1}{(2\pi \sigma_t^2)^{1/2}}\exp\parenthesis{-\frac{\abs{x - \alpha_t \bv^\top \bz}^2+ 1}{\sigma_t^2}}. \label{eq:density-t-history-lower}
    \end{align}
\end{lemma}

\begin{proof}
Assumption \ref{ass:truncated-subGaussian} implies
    \begin{align*}
        \phi^h_{\bz}(t, x) & = \frac{1}{(2\pi \sigma_t^2)^{1/2}}\int  \phi^h_{\bz}(0, y) \exp\parenthesis{-\frac{\abs{x - \alpha_t y}^2}{2\sigma_t^2}}\dd y \\ & \leq \frac{C_1}{(2\pi \sigma_t^2)^{1/2}}\int \exp\parenthesis{-C_2\abs{y - \bv^\top \bz}^2/2}\exp\parenthesis{-\frac{\abs{x - \alpha_t y}^2}{2\sigma_t^2}}\dd y.
    \end{align*}
Completing the square to have
\begin{align*}
    & -\frac{C_2\abs{y - \bv^\top \bz}^2}{2}  -\frac{\abs{x - \alpha_t y}^2}{2\sigma_t^2} \\ & = y^2\parenthesis{-\frac{C_2}{2} - \frac{\alpha_t^2}{2\sigma_t^2}} + y\parenthesis{C_2 \bv^\top \bz + \frac{\alpha_t x}{\sigma_t^2}} -\frac{C_2}{2}\bz^\top v\bv^\top \bz - \frac{(x)^2}{2\sigma_t^2} \\ & = -\frac{C_2 \sigma_t^2 + \alpha_t^2}{2\sigma_t^2}\parenthesis{y - \frac{\sigma_t^2}{C_2\sigma_t^2 + \alpha_t^2}\parenthesis{C_2 \bv^\top \bz + \frac{\alpha_t x}{\sigma_t^2} }}^2  -\frac{C_2}{2(C_2\sigma_t^2 + \alpha_t^2)}(x - \alpha_t \bv^\top \bz)^2.
\end{align*}
Consequently, we have 
\begin{align*}
     \phi^h_{\bz}(t, x) & \leq \frac{C_1}{(2\pi \sigma_t^2)^{1/2}}\exp\parenthesis{-\frac{C_2}{2(C_2\sigma_t^2 + \alpha_t^2)}(x - \alpha_t \bv^\top \bz)^2} \\ & \times \int\exp\parenthesis{-\frac{C_2 \sigma_t^2 + \alpha_t^2}{2\sigma_t^2}\parenthesis{y - \frac{\sigma_t^2}{C_2\sigma_t^2 + \alpha_t^2}\parenthesis{C_2 \bv^\top \bz + \frac{\alpha_t x}{\sigma_t^2} }}^2} \dd y \\ & = \frac{C_1}{\parenthesis{C_2\sigma_t^2 + \alpha_t^2}^{1/2}}\exp\parenthesis{-\frac{C_2}{2(C_2\sigma_t^2 + \alpha_t^2)}(x - \alpha_t \bv^\top \bz)^2},
\end{align*}
which completes the upper bound. 

For the lower bound, we have
\begin{align*}
    \phi^h_{\bz}(t, x)  & = \frac{1}{(2\pi \sigma_t^2)^{1/2}}\int  \phi^h_{\bz}(0, y) \exp\parenthesis{-\frac{\abs{x - \alpha_t y}^2}{2\sigma_t^2}}\dd y \\ & \geq \frac{1}{(2\pi \sigma_t^2)^{1/2}}\int_{\abs{y - \bv^\top \bz} \leq R}  \phi^h_{\bz}(0, y) \exp\parenthesis{-\frac{\abs{x - \alpha_t y}^2}{2\sigma_t^2}}\dd y \\ & \geq \frac{1}{(2\pi \sigma_t^2)^{1/2}}\int_{\abs{y - \bv^\top \bz} \leq R}  \phi^h_{\bz}(0, y) \exp\parenthesis{-\frac{\abs{x - \alpha_t (y - \bv^\top \bz) - \alpha_t \bv^\top \bz}^2}{2\sigma_t^2}}\dd y \\ & \geq \frac{1}{(2\pi \sigma_t^2)^{1/2}}\exp\parenthesis{-\frac{\abs{x - \alpha_t \bv^\top \bz}^2}{\sigma_t^2}}\int_{\abs{y - \bv^\top \bz} \leq R}\phi^h_{\bz}(0, y)\exp\parenthesis{-\frac{\alpha_t^2\abs{y - \bv^\top \bz}^2}{\sigma_t^2}} \dd y \\ & \geq \frac{1}{(2\pi \sigma_t^2)^{1/2}}\exp\parenthesis{-\frac{\abs{x - \alpha_t \bv^\top \bz}^2}{\sigma_t^2}}\exp\parenthesis{-\frac{\alpha_t^2 R^2}{\sigma_t^2}}\int_{\abs{y - \bv^\top \bz} \leq R}\phi^h_{\bz}(0, y) \dd y.
\end{align*}
Moreover, Assumption \ref{ass:truncated-subGaussian} implies
\begin{align*}
    \int_{\abs{y - \bv^\top \bz} \leq R}\phi^h_{\bz}(0, y) \dd y   & = 1 -  \int_{\abs{y - \bv^\top \bz} > R}\phi^h_{\bz}(0, y) \dd y \\ & \geq 1 - C_1 \int_{\abs{z} > R}\exp(-C_2 z^2/2) \dd z \\ & \geq 1 - C_2' R^{-1}\exp(-C_2R^2/2),
\end{align*}
where $C_2' = \frac{2\pi^{1/2}}{C_2\Gamma(3/2)}$ \citep[Lemma 16]{chen2023score}. Finally, we choose $R$ large enough to conclude~that
\begin{align*}
    \phi^h_{\bz}(t, x) \gtrsim \frac{1}{(2\pi \sigma_t^2)^{1/2}}\exp\parenthesis{-\frac{\abs{x - \alpha_t \bv^\top \bz}^2+ 1}{\sigma_t^2}}.
\end{align*}

\end{proof}

\begin{lemma}\label{lemma:clip integral}
    Under Assumption \ref{ass:truncated-subGaussian}, for any $\nu \in \mathbb{Z}_+$ and $\epsilon > 0$, there exist constants $R_2$ and $R_3$ satisfying
    \[
    R_2 \gtrsim \max(\sqrt{\log(1/\epsilon)}, \sqrt{\nu}),
    \qquad
    R_3 \gtrsim \sqrt{\log(1/\epsilon)}
    \]
    such that for any $x \in \R$ and $\bz \in \R^{h-k}$ the following holds:
    \begin{align*}
    \end{align*}
    where we define
    \begin{align}
        \bar{B} = \bracket{\frac{x - \sigma_t R_2}{\alpha_t}, \frac{x + \sigma_t R_2}{\alpha_t}} \bigcap \bracket{\bv^\top \bz - R_3, \bv^\top \bz + R_3}, \label{eq:truncation region}
    \end{align}
\end{lemma}

\begin{proof}
    We consider two truncation regions:
    \begin{align*}
        &B_{1} \coloneqq \curly{y \in \R: \abs{\frac{\alpha_t y - x}{\sigma_t}} \leq R_2} = \bracket{\frac{x - \sigma_t R_2}{\alpha_t}, \frac{x + \sigma_t R_2}{\alpha_t}},\\ 
        & B_{2}  \coloneqq \curly{y \in \R: \abs{y - \bv^\top \bz} \leq R_3} = \bracket{\bv^\top \bz - R_3, \bv^\top \bz + R_3}.
    \end{align*}
   It follows $\bar{B} = B_{1} \cap B_{2}$. We bound the integral with
    \begin{align}
        & \frac{1}{\sigma_t(2\pi)^{1/2}}\int_{\R \setminus \bar{B}}  \abs{\parenthesis{\frac{\alpha_t y - x}{\sigma_t}}^\nu }\phi^h_{\bz}(0, y) \exp\parenthesis{-\frac{\abs{x - \alpha_t y}^2}{2\sigma_t^2}}\dd y \nonumber \\ & \leq \underbrace{\frac{1}{\sigma_t(2\pi)^{1/2}}\int_{\R \setminus B_{1}}  \abs{\parenthesis{\frac{\alpha_t y - x}{\sigma_t}}^\nu }\phi^h_{\bz}(0, y) \exp\parenthesis{-\frac{\abs{x - \alpha_t y}^2}{2\sigma_t^2}}\dd y}_{(A_1)} \nonumber \\ & \qquad + \underbrace{\frac{1}{\sigma_t(2\pi)^{1/2}}\int_{\R \setminus B_{2}}  \abs{\parenthesis{\frac{\alpha_t y - x}{\sigma_t}}^\nu }\phi^h_{\bz}(0, y) \exp\parenthesis{-\frac{\abs{x - \alpha_t y}^2}{2\sigma_t^2}}\dd y}_{(A_2)}. \label{eq:split int}
    \end{align}
    Over each region, we bound the integral in two ways. 
    Note that for all $R_2 > \sqrt{\nu}$, monotonicity implies
    \begin{align*}
        \abs{\parenthesis{\frac{\alpha_t y - x}{\sigma_t}}^\nu } \exp\parenthesis{-\frac{\abs{x - \alpha_t y}^2}{2\sigma_t^2}} \leq R_2^\nu \exp(-R_2^2/2).
    \end{align*}
    Thus, we bound the term $(A_1)$ with
    \begin{align}
        (A_1) \leq \frac{R_2^\nu \exp(-R_2^2/2)}{\sigma_t(2\pi)^{1/2}}\int_{\R\setminus B_{1}}\phi^h_{\bz}(0, y)\dd y \lesssim \frac{R_2^\nu \exp(-R_2^2/2)}{\sigma_t}. \label{eq:a1-1}
    \end{align}
    On the other hand, change of variable $w = \alpha_t y - x/\sigma_t$ leads to 
    \begin{align}
        (A_1) \leq \frac{C_1}{\alpha_t(2\pi)^{1/2}}\int_{\abs{w} > R_2}w^\nu \exp(-w^2/2) \dd w \lesssim \frac{R_2^{\nu - 1}\exp(-R_2^2 / 2)}{\alpha_t}.  \label{eq:a1-2}
    \end{align}
    Combining \eqref{eq:a1-1} and \eqref{eq:a1-2}, we deduce that 
    \begin{align}
        (A_1) \lesssim R_2^{\nu} \exp(-R_2^2/2). \label{eq:a1-3}
    \end{align}
    Similarly, we have the following two bounds for $(A_2)$:
    \begin{align}
        (A_2) & \leq \frac{C_1}{\sigma_t(2\pi)^{1/2}}\exp(-C_2 R_3^2/2)\int_{\R \setminus B_{2}}\abs{\parenthesis{\frac{\alpha_t y - x}{\sigma_t}}^\nu } \exp\parenthesis{-\frac{\abs{x - \alpha_t y}^2}{2\sigma_t^2}}\dd y \nonumber \\ & \leq \frac{C_1}{\alpha_t(2\pi)^{1/2}}\exp(-C_2 R_3^2/2) \int_\R w^\nu \exp(-w^2/2) \dd w \lesssim \frac{1}{\alpha_t}\exp(-C_2 R_3^2/2), \label{eq:a2-1}
    \end{align}
    and 
    \begin{align}
        (A_2) \leq \frac{C_1}{\sigma_t (2\pi)^{1/2}}\parenthesis{\frac{v}{e}}^{v/2}\int_{\R \setminus B_{2}}\exp(-C_2\abs{y - \bv^\top \bz}^2/2) \dd y \lesssim \frac{1}{\sigma_t}\exp(-C_2 R_3^2/2). \label{eq:a2-2}
    \end{align}
    Combining \eqref{eq:a2-1} and \eqref{eq:a2-2} leads to
    \begin{align}
        (A_2) \lesssim \exp(-C_2 R_3^2/2).  \label{eq:a2-3}
    \end{align}
    Consequently, adding up \eqref{eq:a1-3} and \eqref{eq:a2-3} gives us
    \begin{align*}
        (A_1) + (A_2) \lesssim (R_2)^\nu \exp(-R_2^2/2) + \exp(-C_2 R_3^2/2) \lesssim \epsilon,
    \end{align*}
    as long as $R_2, R_3 \gtrsim \sqrt{\log(1/\epsilon)}$. 
    
\end{proof}

\begin{lemma}\label{lemma:score-bound}
    Under Assumption \ref{ass:truncated-subGaussian}, the following inequality holds:
    \begin{align*}
        \abs{\partial_x \log\phi^h_{\bz}(t, x)} \lesssim \frac{1}{\sigma_t^2}\parenthesis{\abs{x - \alpha_t \bv^\top \bz} + 1}.
    \end{align*}
\end{lemma}

\begin{proof}
    We write the score function as
    \begin{align}
        \partial_x \log\phi^h_{\bz}(t, x) = \frac{\int  \frac{\alpha_t y - x}{\sigma_t^2}\phi^h_{\bz}(0, y) \exp\parenthesis{-\frac{\abs{x - \alpha_t y}^2}{2\sigma_t^2}}\dd y }{\int  \phi^h_{\bz}(0, y) \exp\parenthesis{-\frac{\abs{x - \alpha_t y}^2}{2\sigma_t^2}}\dd y}. \label{eq:score-ratio}
    \end{align}
    By taking $\nu \in \curly{0, 1}$ in Lemma \ref{lemma:clip integral}, for any $\epsilon > 0$ to be chosen later, there are sufficiently large constants $R_2$ and $R_3$ such that
    \begin{align}
        & \int_{\R \setminus B}  \abs{\frac{\alpha_t y - x}{\sigma_t^2}}\phi^h_{\bz}(0, y) \exp\parenthesis{-\frac{\abs{x - \alpha_t y}^2}{2\sigma_t^2}}\dd y \leq \epsilon, \label{eq:clip-int-1} \\
        & \int_{\R \setminus B}  \phi^h_{\bz}(0, y) \exp\parenthesis{-\frac{\abs{x - \alpha_t y}^2}{2\sigma_t^2}}\dd y \leq \epsilon.  \label{eq:clip-int-0}
    \end{align}
    where $B$ is defined as in \eqref{eq:truncation region}. Lemma \ref{lemma:tail-density-t} implies $\phi^h_{\bz}(t, x) \geq 2\epsilon$ with $\epsilon \asymp \frac{1}{\sigma_t}\exp\parenthesis{-\frac{\abs{x - \alpha_t \bv^\top \bz}^2 + 1}{\sigma_t^2}}$. It follows from \eqref{eq:clip-int-0} that
    \begin{align}
        \int_{B}  \phi^h_{\bz}(0, y) \exp\parenthesis{-\frac{\abs{x - \alpha_t y}^2}{2\sigma_t^2}}\dd y \geq \epsilon. \label{eq:clip-int-inner-0}
    \end{align}
    Combining \eqref{eq:score-ratio}--\eqref{eq:clip-int-inner-0}, we obtain
    \begin{align*}
        \abs{\partial_x \log\phi^h_{\bz}(t, x)} & \leq \frac{\int_{\R}  \abs{\frac{\alpha_t y - x}{\sigma_t^2}}\phi^h_{\bz}(0, y) \exp\parenthesis{-\frac{\abs{x - \alpha_t y}^2}{2\sigma_t^2}}\dd y }{\int_{B}  \phi^h_{\bz}(0, y) \exp\parenthesis{-\frac{\abs{x - \alpha_t y}^2}{2\sigma_t^2}}\dd y} \\ & \leq \frac{\int_{B}  \abs{\frac{\alpha_t y - x}{\sigma_t^2}}\phi^h_{\bz}(0, y) \exp\parenthesis{-\frac{\abs{x - \alpha_t y}^2}{2\sigma_t^2}}\dd y }{\int_{B}  \phi^h_{\bz}(0, y) \exp\parenthesis{-\frac{\abs{x - \alpha_t y}^2}{2\sigma_t^2}}\dd y} + \frac{\epsilon}{\epsilon} \leq \frac{R_2}{\sigma_t} + 1. 
    \end{align*}
    Let $R_2 \asymp \sqrt{\log(1/\epsilon)}$. We conclude that
    \begin{align*}
        \abs{\partial_x \log\phi^h_{\bz}(t, x)} \lesssim \frac{1}{\sigma_t^2}\parenthesis{\abs{x - \alpha_t \bv^\top \bz} + 1}.
    \end{align*}
\end{proof}

\begin{lemma}\label{lemma:score-tail}
    Suppose Assumption \ref{ass:truncated-subGaussian} holds. There is a constant $C_6 > 0$, such that for any $R_1>0$ and $t>0$  we have
    \begin{align}  \E\bracket{1\curly{\norm{\bX_0^{[k:h)}}_\infty > R_1}\abs{\partial_x \log \phi^h_{\bX_0^{[k:h)}}(t, X_t^h)}^2} \lesssim  \frac{1}{\sigma_t^4} \exp(-C_6 R_1^2/2). \label{eq:E1-2}
    \end{align}
\end{lemma}

\begin{proof}
Given $\bz$, Lemma \ref{lemma:score-bound} leads to 
\begin{align}
    \E\bracket{\abs{\partial_x \log\phi^h_{\bX_0^{[k:h)}}(t, X_t^h)}^2\big\vert \bX_0^{[k:h)}} \lesssim \frac{1}{\sigma_t^4}\parenthesis{1 + \E\bracket{\abs{X_t^h - \alpha_t \bv^\top \bX_0^{[k:h)}}^2\big\vert \bX_0^{[k:h)}}}. \label{eq:score-tail-1}
\end{align}
With \eqref{eq:density-t-history-upper} and change of variable $z = x - \alpha_t \bv^\top \bz$, we bound the conditional expectation with
\begin{align*}
    \E\bracket{\abs{X_t^h - \alpha_t \bv^\top \bX_0^{[k:h)}}^2\big\vert \bX_0^{[k:h)}} & \leq \frac{C_1}{\parenthesis{C_2\sigma_t^2 + \alpha_t^2}^{1/2}} \int z^2 \exp\parenthesis{-\frac{C_2}{2(C_2\sigma_t^2 + \alpha_t^2)}z^2} \dd z \\ & \leq \frac{2C_1}{C_2}(C_2\sigma_t^2 + \alpha_t^2)^{1/2} \lesssim 1. 
\end{align*}
Consequently, Eq.~\eqref{eq:score-tail-1} becomes
\begin{align*}
    \E\bracket{\abs{\partial_x \log\phi^h_{\bX_0^{[k:h)}}(t, X_t^h)}^2\big\vert \bX_0^{[k:h)}}  \lesssim \frac{1}{\sigma_t^4}.
\end{align*}
Taking expectation over $\bX_0^{[k:h)}$, we conclude that
\begin{align*}    & \E\bracket{1\curly{\norm{\bX_0^{[k:h)}}_\infty > R_1}\abs{\partial_x \log \phi^h_{\bX_0^{[k:h)}}(t, X_t^h)}^2} \\ &  = \E\bracket{1\curly{\norm{\bX_0^{[k:h)}}_\infty > R_1}\E\bracket{\abs{\partial_x \log\phi^h_{\bX_0^{[k:h)}}(t, X_t^h)}^2\big\vert \bX_0^{[k:h)}}} \lesssim \frac{1}{\sigma_t^4} \exp(-C_6 R_1^2/2),
\end{align*}
for some constant $C_6 > 0$ as in Corollary \ref{coro:joint-subG}.
    
\end{proof}

\begin{lemma}\label{lemma:E2-terms}
    Suppose Assumption \ref{ass:truncated-subGaussian} holds. For any $R_2>0$, $\bz \in \R^{h-k}$, $t>0$, we have
    \begin{align*}
        & \int_{\abs{x - \alpha_t \bv^\top \bz}> R_2} \phi^h_{\bz}(t, x) \dd x  \lesssim (R_2)^{-1}\exp\parenthesis{-C_2'R_2^2/2},\\ 
        & \int_{\abs{x - \alpha_t \bv^\top \bz}> R_2} \phi^h_{\bz}(t, x)\abs{\partial_x \log \phi^h_{\bz}(t, x)}^2 \dd x  \lesssim \frac{1}{\sigma_t^4}R_2\exp\parenthesis{-C_2'R_2^2/2},
    \end{align*}
    where $C_2' = \min_{t>0}\frac{C_2}{C_2\sigma_t^2 + \alpha_t^2} = \frac{C_2}{\max(C_2, 1)}$.
\end{lemma}

\begin{proof}
    Lemma \ref{lemma:tail-density-t} implies
    \begin{align*}
         & \int_{\abs{x - \alpha_t \bv^\top \bz}> R_2} \phi^h_{\bz}(t, x) \dd x \\ & \leq  \frac{C_1}{\parenthesis{C_2\sigma_t^2 + \alpha_t^2}^{1/2}}\int_{\abs{x - \alpha_t \bv^\top \bz}> R_2}\exp\parenthesis{-\frac{C_2}{2(C_2\sigma_t^2 + \alpha_t^2)}(x - \alpha_t \bv^\top \bz)^2} \dd x \\ & = \frac{C_1}{\parenthesis{C_2\sigma_t^2 + \alpha_t^2}^{1/2}}\int_{\abs{u}> R_2}\exp\parenthesis{-\frac{C_2}{2(C_2\sigma_t^2 + \alpha_t^2)}u^2} \dd u \\ & \leq \frac{2\sqrt{\pi}C_1(C_2\sigma_t^2 + \alpha_t^2)^{1/2}}{C_2\Gamma(3/2)R_2}\exp\parenthesis{-\frac{C_2 R_2^2}{2(C_2\sigma_t^2 + \alpha_t^2)}} \lesssim (R_2)^{-1}\exp\parenthesis{-C_2'R_2^2/2},
    \end{align*}
    where we invoke \cite[Lemma 16]{chen2023score} in the second inequality with $d = 1$ and we set $C_2' = \min_{t>0}\frac{C_2}{C_2\sigma_t^2 + \alpha_t^2}$. Similarly, Lemma \ref{lemma:score-bound} implies
    \begin{align*}
         & \int_{\abs{x - \alpha_t \bv^\top \bz}> R_2} \phi^h_{\bz}(t, x)\abs{\partial_x \log \phi^h_{\bz}(t, x)}^2 \dd x \\ & \lesssim \frac{C_1}{\sigma_t^4(C_2\sigma_t^2 + \alpha_t^2)^{1/2}}\int_{\abs{z} > R_2}(1 + z^2) \exp\parenthesis{-\frac{C_2}{2(C_2\sigma_t^2 + \alpha_t^2)}z^2}\dd z \\ & \leq \frac{C_1}{\sigma_t^4(C_2\sigma_t^2 + \alpha_t^2)^{1/2}}\frac{C_2\sigma_t^2 + \alpha_t^2}{C_2}(R_2^{-1} + R_2)\exp\parenthesis{-\frac{C_2 R_2^2}{2(C_2\sigma_t^2 + \alpha_t^2)}} \\ & \lesssim \frac{1}{\sigma_t^4}R_2 \exp\parenthesis{-\frac{C_2 R_2^2}{2(C_2\sigma_t^2 + \alpha_t^2)}} \lesssim\frac{1}{\sigma_t^4}R_2\exp\parenthesis{-C_2'R_2^2/2}.
    \end{align*}
\end{proof}

\begin{lemma}\label{lemma:small density}
    Suppose Assumption \ref{ass:truncated-subGaussian} holds. For any $R_6 \geq 1, \epsilon_{\rm low} > 0$, $\bz \in \R^{h - k}$  and $t>0$, we have
    \begin{align*}
        & \int_{\abs{x - \alpha_t \bv^\top \bz} \leq R_6}1\curly{\abs{\phi^h_{\bz}(t, x)} < \epsilon_{\rm low}} \phi^h_{\bz}(t, x) \dd x  \lesssim R_6 \epsilon_{\rm low}, \\ 
        & \int_{\abs{x - \alpha_t \bv^\top \bz} \leq R_6} 1\curly{\abs{\phi^h_{\bz}(t, x)} < \epsilon_{\rm low}} \phi^h_{\bz}(t, x)\abs{\partial_x \log \phi^h_{\bz}(t, x)}^2 \dd x  \lesssim \frac{\epsilon_{\rm low}}{\sigma_t^4}R_6^3.
    \end{align*}
\end{lemma}

\begin{proof}
    Direct calculation implies
    \begin{align*}
        & \int_{\abs{x - \alpha_t \bv^\top \bz} \leq R_6}1\curly{\abs{\phi^h_{\bz}(t, x)} < \epsilon_{\rm low}} \phi^h_{\bz}(t, x) \dd x  \\ & \leq \int_{\abs{x - \alpha_t \bv^\top \bz} \leq R_6} \epsilon_{\rm low} \dd x = 2R_6\epsilon_{\rm low}.
    \end{align*}
    Similarly, Lemma \ref{lemma:score-bound} implies
    \begin{align*}
        & \int_{\abs{x - \alpha_t \bv^\top \bz} \leq R_6} 1\curly{\abs{\phi^h_{\bz}(t, x)} < \epsilon_{\rm low}} \phi^h_{\bz}(t, x)\abs{\partial_x \log \phi^h_{\bz}(t, x)}^2 \dd x \\ & \lesssim \frac{\epsilon_{\rm low}}{\sigma_t^4}\int_{\abs{x - \alpha_t \bv^\top \bz} \leq R_6}\parenthesis{\abs{x - \alpha_t \bv^\top \bz} + 1}^2 \dd x \lesssim \frac{\epsilon_{\rm low}}{\sigma_t^4}R_6^3.
    \end{align*}
\end{proof}

\end{document}